\documentclass[lettersize,journal]{IEEETai}
\let\labelindent\relax
\usepackage[utf8]{inputenc}
\usepackage[T1]{fontenc}
\usepackage{times}
\usepackage{epsfig}
\usepackage{graphicx}
\usepackage{yfonts}
\usepackage{xcolor}
\usepackage{hyperref}
\usepackage[all]{hypcap}
\usepackage{booktabs}
\usepackage[normalem]{ulem}
\usepackage{fancyhdr}
\useunder{\uline}{\ul}{}
\usepackage{lipsum}
\usepackage{bm,nicefrac}
\usepackage{amsmath,mathrsfs,amssymb} 
\usepackage{amsfonts}  
\usepackage{algorithm}
\usepackage{algorithmic}
\usepackage{array}
\usepackage{textcomp}
\usepackage{url}
\usepackage{verbatim}
\usepackage{amsmath,mathrsfs,amssymb} 
\usepackage{amsfonts}  
\usepackage{algorithm}
\usepackage{algorithmic}

\usepackage{siunitx}   
\usepackage{upgreek}   
\usepackage{xcolor}    
\usepackage{float}
\usepackage{subfig}
\usepackage{booktabs}       
\usepackage{threeparttable}
\usepackage{multirow}
\usepackage{multicol}
\usepackage{times}
\usepackage{hyperref}
\usepackage{bm}
\usepackage{makecell}
\usepackage{footnote}
\usepackage{ragged2e}
\usepackage{amsthm}
\usepackage{dsfont}
\usepackage{lipsum} 
\usepackage[sort&compress]{natbib}
\setcitestyle{numbers,square}

\newtheorem{theorem}{Theorem}
\newtheorem{lemma}{Lemma}

\usepackage{enumitem} 
\usepackage{pifont}
\newcommand{\cmark}{\ding{51}}%
\newcommand{\xmark}{\ding{55}}%

\def\x{\boldsymbol{x}}
\def\v{\boldsymbol{v}}

\theoremstyle{definition}

\usepackage{cuted}

\def\w{\boldsymbol{w}}
\def\W{\boldsymbol{W}}

\newcommand{\fl}[1] {\textcolor{black}{#1}}

\begin{document}

\title{Quadratic Neuron-empowered Heterogeneous Autoencoder for Unsupervised Anomaly Detection}

\author{Jing-Xiao~Liao, \textit{Graduate Student Member, IEEE}, Bo-Jian Hou, Hang-Cheng Dong, Hao Zhang \\
Xiaoge Zhang,  \textit{Member, IEEE},
Jinwei Sun, Shiping Zhang, Feng-Lei Fan, \textit{Member, IEEE}

\thanks{The work described in this paper was partially supported by the Direct Grant for Research from the Chinese University of Hong Kong and ITS/173/22FP from the Innovation and Technology Fund of Hong Kong, was partially supported by a grant from the Research Grants Council of the Hong Kong Special Administrative Region, China (Project No. PolyU 25206422), and was partially supported by the Research Committee of The Hong Kong Polytechnic University under project code G-UAMR and student account code RL3C. \textit{(Corresponding Authors: Shiping Zhang, Feng-Lei Fan)}}
\thanks{Jing-Xiao Liao is with School of Instrumentation Science and Engineering, Harbin Institute of Technology, Harbin, China, and is also with Department of Industrial and Systems Engineering, The Hong Kong Polytechnic University, Hong Kong, SAR of China}
\thanks{Hang-Cheng Dong, Jinwei Sun, Shiping Zhang are with School of Instrumentation Science and Engineering, Harbin Institute of Technology, Harbin, China. (email: spzhang@hit.edu.cn)}
\thanks{Bo-Jian Hou, Hao Zhang are with Weill Cornell Medicine, Cornell University, New York City, US.}
\thanks{Xiaoge Zhang is with Department of Industrial and Systems Engineering, The Hong Kong Polytechnic University, Hong Kong, SAR of China.}
\thanks{Feng-Lei Fan is with Department of Mathematics, The Chinese University of Hong Kong, Hong Kong, SAR of China. (email: hitfanfenglei@gmail.com)}}



\maketitle

\begin{abstract}
Inspired by the complexity and diversity of biological neurons, a quadratic neuron is proposed to replace the inner product in the current neuron with a simplified quadratic function. Employing such a novel type of neurons offers a new perspective on developing deep learning. When analyzing quadratic neurons, we find that there exists a function such that a heterogeneous network can approximate it well with a polynomial number of neurons but a purely conventional or quadratic network needs an exponential number of neurons to achieve the same level of error. Encouraged by this inspiring theoretical result on heterogeneous networks, we directly integrate conventional and quadratic neurons in an autoencoder to make a new type of heterogeneous autoencoders. To our best knowledge, it is the first heterogeneous autoencoder that is made of different types of neurons. {Next, we apply the proposed heterogeneous autoencoder to unsupervised anomaly detection for tabular data and bearing fault signals. The anomaly detection faces difficulties such as data unknownness, anomaly feature heterogeneity, and feature unnoticeability, which is suitable for the proposed heterogeneous autoencoder. Its high feature representation ability can characterize a variety of anomaly data (heterogeneity), discriminate the anomaly from the normal (unnoticeability), and accurately learn the distribution of normal samples (unknownness).} Experiments show that heterogeneous autoencoders perform competitively compared to other state-of-the-art models.
\end{abstract}

\begin{IEEEImpStatement}
In this work, we present a novel perspective on neural network design by introducing neuronal diversity. Firstly, we prove that neural networks that combine two types of neurons have superior theoretical approximation efficiency compared to networks comprising solely homogeneous neurons. Secondly, we develop the first autoencoder that is made of different neurons, which is a new addition to the family of autoencoders. Lastly, the proposed heterogeneous autoencoder performs better than the state-of-the-art models in tabular and bearing fault unsupervised anomaly detection. In brief, both theoretical derivation and experimental results confirm the benefit of neuronal diversity in neural network design.

\end{IEEEImpStatement}


\section{Introduction}
\IEEEPARstart{D}{eep} learning has achieved great success in many important fields \cite{deeplearning}. However, so far, deep learning innovations mainly revolve around structure design such as increasing depth and proposing novel shortcut architectures \cite{densenet,9614997,10197463}. Almost exclusively, the existing advanced networks only involve the same type of neurons characterized by i) the inner product of the input and the weight vector; ii) a nonlinear activation function. For convenience, we call this type of neurons \textit{conventional neurons}.
Although these neurons can be interconnected to approximate a general function, we argue that the type of neurons useful for deep learning should not be unique. As we know, neurons in a biological neural system are diverse, collaboratively generating all kinds of intellectual behaviors. Since a neural network was initially devised to imitate a biological neural system, neuronal diversity should be carefully examined in neural network research.

In this context, a new type of neurons called \textit{quadratic neurons} was recently proposed in \cite{fan2019quadratic}, which replace the inner product in a conventional neuron with a simplified quadratic function. Hereafter, we call a network made of quadratic neurons a \textit{quadratic network} and a network made of conventional neurons a \textit{conventional network}. The superiority of quadratic networks over conventional networks in representation efficiency and capacity was demonstrated in \cite{fan2021expressivity}. Intuitively, nonlinear features ubiquitously exist in real world, and quadratic neurons are empowered with a quadratic function to represent nonlinear features efficiently and effectively. Mathematically, when the rectified linear unit (ReLU) is employed, a conventional model is a piecewise linear function, whereas a quadratic one is a more powerful piecewise nonlinear mapping. Since a quadratic neuron shows great potential, we are encouraged to keep pushing its theory and application boundaries.

Previously, it was shown that there exists a function such that a conventional network needs an exponential number of neurons to approximate, but a quadratic network only needs a polynomial number of neurons to achieve the same level of error \cite{fan2020universal}. Despite the attraction of this construction, with the help of techniques in \cite{eldan2016power}, we find that one can also construct a function that can realize the opposite situation, \textit{i.e.,} a function that is hard to approximate by a quadratic network but easy to approximate by a conventional network (Theorem \ref{thm:core}). Combining these two constructed functions, one can naturally get a function easy to approximate by a heterogeneous network of both quadratic and conventional neurons but hard by a purely conventional or quadratic network, where the difference remains polynomial vs exponential (Theorem \ref{thm:main}). Moreover, the opposite case will not happen because the purely conventional or quadratic is a degenerated case of the heterogeneous network. Such a theoretical derivation makes the employment of heterogeneous models well grounded, implicating that there are at least some tasks pretty suitable for the combination of conventional and quadratic neurons so that an exponential difference is reached. While for the general task, heterogeneous networks' performance will not be worse, since a purely conventional or quadratic network is the special case of a heterogeneous network.

This theoretical derivation agrees with our intuition. There is no one-size-fits-all artificial neuron in deep learning. Both conventional and quadratic neurons cannot perform well on all kinds of tasks. Encouraged by the inspiring theoretical result on heterogeneous networks and the idea of neuronal synergy, here we directly integrate conventional and quadratic neurons in an autoencoder to make a new type of \textit{heterogeneous autoencoders}. Autoencoders \cite{sae, fan2019quadratic, 9536242, 9582824} are widely used in unsupervised learning tasks such as feature extraction \cite{9648028}, denoising\cite{8383709}, anomaly detection \cite{9971462}, and so on. To the best of our knowledge, all the previous mainstream autoencoders use the same type of neurons regardless of their innovations. The proposed heterogeneous autoencoder is the first heterogeneous autoencoder made up of diverse neurons in the family of autoencoders. {Furthermore, we apply the heterogeneous autoencoder to solve the anomaly detection problem. 

Anomaly detection is widely applied in different contexts such as intelligent manufacturing, data management, and intelligent transportation \cite{9537728,10154127}. However, anomaly detection is challenging due to the following issues \cite{pang2021deep}:} 
\begin{itemize}
    
 \item {\textbf{Unknownness.} The anomaly features may include the unknown data, which requires an anomaly detection model to accurately construct the distribution of normal samples and be sensitive to anomalous samples.}

 \item {\textbf{Heterogeneity.} Anomalies are irregular, and different anomalies may display completely different abnormal features. For instance, when detecting network traffic attacks, four types of attacks are all anomalies (DOS, R2L, U2R, Probing).}

 \item {\textbf{Unnoticeability.}  Anomalies are typically rare and are easily overwhelmed by normal instances. Therefore, it is difficult to detect the anomaly. }

 \end{itemize}

{Based on the above analyses, unsupervised anomaly detection highly relies on the models' ability to feature representation. A model with a high feature representation ability can characterize a variety of anomaly data (heterogeneity), discriminate the anomaly from the normal (unnoticeability), and accurately learn the distribution of normal samples (unknownness). Our theory indicates that a heterogeneous network has a more powerful representation ability than a homogeneous network; therefore, a heterogeneous network is suitable to address the anomaly detection problem.} Experiments demonstrate that a heterogeneous autoencoder {achieves competitive performance} relative to other state-of-the-art on anomaly detection datasets and {real-world bearing fault detection.} In summary, our contributions are

\begin{enumerate}
    \item We prove that to approximate a constructed function, a heterogeneous network is more efficient and has a more powerful capability than a homogeneous one. Different from ensemble learning, our result characterizes the ability of heterogeneous networks from the perspective of approximation theory, which is a novel theoretical angle and paves a way for the employment of heterogeneous networks.

    \item To further verify our theory, we develop a first-of-its-kind autoencoder that integrates essentially different types of neurons in one model, which is a methodological contribution to autoencoder design. 
    
    \item Experiments suggest that the proposed heterogeneous autoencoder delivers competitive performance compared to homogeneous autoencoders and other baselines on unsupervised tabular data anomaly detection.

\end{enumerate}

\section{Related Works \label{sec:related works}}

\fl{\textbf{Anomaly Detection (AD)} algorithms strive to identify outliers that deviate significantly from the majority of data~\cite{han2022adbench}. This study concentrates on unsupervised AD, which exclusively utilizes samples from the normal class to construct a model for outlier detection~\cite{10314785}. Broadly, AD algorithms can be categorized into traditional machine learning (shallow) and deep learning-based (deep) methods~\cite {ruff2021unifying}. First, traditional machine learning methods are roughly summarised as follows: i) Classification-based method, one-class support vector machine (OCSVM)~\cite{ocsvm} and support vector data description (SVDD)~\cite{tax2004support}; ii) Probabilistic-based method, Gaussian mixture model (GMM)~\cite{reynolds2009gaussian} and kernel density estimation (KDE)~\cite{parzen1962estimation}; iii) Reconstruction-based method, probabilistic principal component analysis (p-PCA)~\cite{tipping1999probabilistic}; iv) Distance-based method, local outlier factor (LOF)~\cite{breunig2000lof} and empirical cumulative outlier detection (ECOD)~\cite{li2022ecod}.}
 
\fl{On the other hand, deep learning-based methods that are capable of handling larger and high-dimensional data have recently garnered increased attention. Compared to machine learning methods, deep neural network methods demonstrate the superior generalizability to various types of data. Several approaches integrated deep neural networks into machine learning methods, such as DeepSVDD~\cite{ruff2018deep}, deep autoencoding GMM (DAGMM)~\cite{zong2018deep}, and deep autoencoding support vector data descriptor (DASVDD)~\cite{10314785}. These approaches are predicated on certain prior assumptions. Conversely, other approaches adopted pure deep neural networks to construct the end-to-end AD models without requiring any prior knowledge. The primary objective of these approaches is to enhance the performance and generalizability of AD methods across a broad spectrum of data. The most representative methods are autoencoders (AEs)~\cite{yu2024adversarial, yao2024svd} and generative adversarial networks (GANs)~\cite{schlegl2017unsupervised,9850373}, which automatically learn the latent features of normal samples to enlarge their differences from the outliers in the latent space.}

\fl{This paper specifically focuses on autoencoders for anomaly detection. Compared to GANs, AEs can primarily be adopted for anomaly detection and exhibit lower computational complexity. This also facilitates the easier integration of different neurons.}

\fl{\textbf{Autoencoders (AEs)} constitute a vibrant research field in machine learning and have attained significant success in a variety of tasks, including data generation~\cite{10054417,ojo2024topic}, fault diagnosis~\cite{9536242,9648028}, and anomaly detection~\cite{yu2024adversarial, yao2024svd}. Presently, autoencoders, as a class of reconstruction-based nonlinear dimension reduction models, have emerged as the most extensively adopted methods for unsupervised anomaly detection, owing to their lightweight and user-friendly variants~\cite{ruff2021unifying}. The advantages of autoencoders, such as automatic latent feature extraction, reduction of dimensional curses, and potent non-linear representation ability, have been proven beneficial in addressing the challenges associated with anomaly detection. Recent developments in autoencoders can be categorized into two aspects. i) Ensemble learning has been adapted to the autoencoder: Pan \textit{et al.} \cite{pan2021synchronized} proposed a synchronized heterogeneous autoencoder that uses an item-oriented and a user-oriented autoencoder to serve as a teacher and a student, respectively, towards the distillation of both label-level and feature-level knowledge. Some studies combine multiple autoencoders for heterogeneous data, \textit{e.g.}, combining a denoising autoencoder and a recurrent autoencoder for non-sequential and sequential data, respectively \cite{li2018deep}. ii) Some novel metrics to measure the reconstruction error: Hojjati \textit{et al.}~\cite{10314785} proposed an anomaly score that combines the reconstruction error and the distance of the enclosing hypersphere in the latent representation, and constructed an autoencoder called DASVDD. The low-rank and sparse priors were integrated into the loss function of deep autoencoders and employed in the hyperspectral anomaly detection~\cite{10285414}.}

\fl{However, regardless of representation regularization, architecture modification, and direct combination, these autoencoders are built upon the same type of neurons, which are conventional neurons. In contrast, the proposed heterogeneous autoencoders design an autoencoder by replacing neuron types, which provides a fresh perspective on autoencoder design.}

\fl{\textbf{Heterogeneous neural networks}, characterized by the integration of various sub-networks or models, have shown significant potential in tackling complex scenarios such as big data and multi-modality. Several notable works include the heterogeneous Rybak neural network (HRYNN) introduced by Qi et al., which is composed of multiple Rybak neural networks (RYNNs) has demonstrated excellent image enhancement effects, particularly on the image edge details~\cite{qi2021new}. Sadr et al. constructed heterogeneous neural networks by integrating intermediate features from convolutional and recursive neural networks, exhibiting superior efficiency and generalization performance in sentiment analysis~\cite{sadr2020multi}. Another example is a heterogeneous neural network that integrates a hierarchical fused neural fuzzy and neural network, providing effective representation learning of users and items for multiple recommendation problems~\cite{pham2023hierarchical}.}

\fl{On the other hand, some multimodal large language models (MLLMs), despite not being explicitly named “heterogeneous networks”, can be broadly considered as a variant of heterogeneous neural network~\cite{fu2023mme}. These models utilize different sub-networks to handle data of different modalities. For instance, contrastive language-image pre-training (CLIP)~\cite{radford2021learning} and OpenCLIP~\cite{cherti2023reproducible} construct both image and text encoders to train image-text pairs. These large multimodal models have shown powerful performance in various downstream tasks, such as zero-shot classification, retrieval, linear probing, and end-to-end fine-tuning. The concept of heterogeneous networks is also implied in some text-to-image diffusion models. For example, ControlNet injects additional conditions into a neural network block and incorporates it as a branch of the large pre-trained stable diffusion model~\cite{rombach2022high} for fine-tuning~\cite{zhang2023adding}. Another text-conditional image diffusion model, RAPHAEL, is capable of generating highly artistic images through text descriptions. It stacks space and time mixture-of-experts (MoEs) layers, which consist of distinguished neural structures and construct billions of diffusion paths from the input to the output~\cite{xue2024raphael}.}

\fl{In summary, the principle of heterogeneity is prevalent in the design of deep learning networks. Integrating diverse architecture is beneficial in handling complex tasks. This characteristic is also advantageous for unsupervised AD, which requires a strong representation ability to represent distinguished latent features. Our heterogeneous network is different from the above in terms of that the heterogeneity lies in the incorporation of two kinds of neurons, providing superior representation of both linear and quadratic features.}

{\textbf{Polynomial networks.} Recently, higher-order units were revisited in the context of deep learning \cite{zoumpourlis2017non,pnet}. 1) Chrysos \textit{et al.} \cite{pnet, chrysos2021deep} successfully applied the high-order units into a deep network by reducing the complexity of the individual high-order unit via tensor decomposition. They investigated the architecture of polynomial networks and achieved state-of-the-art results in image classification \cite{chrysos2022augmenting}, as well as its robustness property \cite{rocamora2022sound}. To achieve the parametric efficiency, a plethora of polynomial neuron studies focused on quadratic neurons.
Table \ref{tab:qneurons} outlines the latest-proposed quadratic neurons. The complexity of neurons in \cite{zoumpourlis2017non, micikevicius2017mixed,jiang2020nonlinear,mantini2021cqnn} is of $\mathcal{O}(n^2)$, but the quadratic neuron we use has only $\mathcal{O}(3n)$ parameters and scales well when a network becomes deep. The neurons in \cite{goyal2020improved,bu2021quadratic,xu2022quadralib, pnet} are a special case of \cite{fan2019quadratic}, which means that this neuron is expressive. Therefore, we think that the quadratic neuron we use achieves an excellent balance between scalability and expressivity.
2) Neurons with polynomial activation \cite{livni2014computational} can also lead to a polynomial network. However, some experiments \cite{fan2021expressivity} show that using polynomial activation often causes a network not to converge and cannot obtain meaningful results.} \\

\begin{table}[htbp]
\centering
\caption{A summary of the latest-proposed quadratic neurons. $\sigma(\cdot)$ denotes the nonlinear activation function. $\odot$ denotes Hadamard product. Here, $\W \in \mathbb{R}^{n\times n}$, $\w, \w^r, \w^g, \w^b \in \mathbb{R}^{n\times 1}$, and we omit the bias terms in these neurons.}
\scalebox{0.94}{
\begin{tabular}{ll}
\toprule
Papers           & Formulations          \\ \midrule
Zoumpourlis \textit{et al.}(2017) \cite{zoumpourlis2017non, micikevicius2017mixed} & $\sigma(\x^{\top}\W\x+\x^\top\w )$               \\ 
Jiang \textit{et al.}(2019) \cite{jiang2020nonlinear}      & \multirow{2}{*}{$\sigma(\x^{\top}\W\x$)} \\ 
Mantini\&Shah(2021) \cite{mantini2021cqnn}     &                        \\ 
Goyal \textit{et al.}(2020) \cite{goyal2020improved}    & $\sigma((\x\odot\x)^\top \w)$               \\ 
Bu\&Karpatne(2021) \cite{bu2021quadratic}       & $\sigma((\x^\top \w^r)\cdot(\x^\top \w^g))$             \\ 
Xu \textit{et al.}(2022) \cite{xu2022quadralib} & $\sigma((\x^\top \w^r)\cdot (\x^\top \w^g)+\x^\top\w^b)$ \\

Fan \textit{et al.}(2018)   \cite{fan2018new}      & $\sigma((\boldsymbol{x}^\top\boldsymbol{w}^{r})(\boldsymbol{x}^\top\boldsymbol{w}^{g})+(\boldsymbol{x}\odot\boldsymbol{x})^\top\boldsymbol{w}^{b})$        \\ \bottomrule
\end{tabular}}
\label{tab:qneurons}
\end{table}

There exist non-linear CNN models that use bilinear or trilinear attention modules, which are popular in image classification and reconstruction tasks \cite{hu2019vision, kong2017low}. These models enhance their representation abilities by devising novel non-linear attention modules or layers. The equations for these attention modules resemble polynomial neurons, such as $\x^\top \W\x$ \cite{yu2018hierarchical} or $((\W\x^\top)\x)\x$ \cite{zheng2019looking}. For example, Sun et al. proposed a higher-order polynomial transformer (HP-Transformer) for fine-grained freezing of gait patterns \cite{10100692}. HP-Transformer focuses on high-order gait patterns. However, they differ fundamentally from polynomial neurons. The attention modules implement non-linear operations through the sum and product of convolution layers, which provide a global attention mechanism. Polynomial neurons are point-wise and provide a local self-attention mechanism \cite{liao2022attention}. This local attention is efficient and lightweight in extracting latent features from tabular datasets compared to attention modules. Therefore, we develop quadratic neurons for anomaly detection tasks.

{This work is essentially different from our previous studies. In \cite{fan2019quadratic}, we introduced a quadratic convolutional autoencoder to address the issue of denoising in low-dose computed tomography. \cite{fan2021expressivity} highlighted the expressivity of quadratic networks by using the spline theory and a measure from algebraic geometry, and presented an effective training method called ReLinear. \cite{fan2020universal} established a theorem that demonstrates the superior approximation capability of quadratic networks compared to conventional networks in some functions, while the main theorem in this draft provides a theory for heterogeneous networks. These earlier studies did not touch on the idea of synergizing different neurons into a network and the efficiency of heterogeneous networks.}




\section{Power of Heterogeneous Networks}
Here, we prove a non-trivial theorem that there exists a function that is easy to approximate by a heterogeneous network but hard to approximate by a purely conventional or quadratic network, \textit{i.e.}, the difference in the number of neurons needed for the same level of error is polynomial vs exponential. This result suggests that at least for some tasks, combining conventional and quadratic neurons is instrumental to model's power and efficiency, which makes the employment of heterogeneous networks well-supported.

\textbf{Quadratic neuron \cite{fan2019quadratic}.} A quadratic neuron computes two inner products and one power term of the input vector and aggregates them for a nonlinear activation function. Mathematically, given the {$d$-dimension} input $\boldsymbol{x} = (x_1, x_2, \ldots ,x_d) \in \mathbb{R}^{d\times1}$, the output of a quadratic neuron is expressed as 
\begin{equation}
\begin{aligned}
 &\sigma\Big((\sum_{i=1}^{d} w_{i}^r x_i +b^r)(\sum_{i=1}^{d} w_{i}^g x_i +b^g) + \sum_{i=1}^{d} w_{i}^b x_{i}^2+c \Big) \\
=&\sigma\Big((\boldsymbol{x}^\top\boldsymbol{w}^{r}+b^{r})(\boldsymbol{x}^\top\boldsymbol{w}^{g}+b^{g})+(\boldsymbol{x}\odot\boldsymbol{x})^\top\boldsymbol{w}^{b}+c\Big),
\end{aligned}
\label{qneq1}
\end{equation}
where $\sigma(\cdot)$ is a nonlinear activation function, $\odot$ denotes the Hadamard product, $\boldsymbol{w}^r,\boldsymbol{w}^g, \boldsymbol{w}^b\in\mathbb{R}^{d\times 1}$ are weight vectors, and $b^r, b^g, c\in\mathbb{R}$ are biases. Please note that superscripts $r,g,b$ are just marks for convenience without special meanings.


\begin{table}[h]
\caption{A table of important symbols}
{
\scalebox{0.9}{\begin{tabular}{l|l}
\toprule
 Symbol                           & Description                                              \\ \midrule
$\boldsymbol x$                  & Input                                               \\
$d$                              & The dimension of the input                          \\ 
$\boldsymbol{w}^{r}$, $\boldsymbol{w}^{g}$, $\boldsymbol{w}^{b}$   & Weight vectors in a quadratic neuron \\
$b^r, b^g, c$ & Biases \\
$\sigma$ & The activation function \\
$\boldsymbol{\omega}$ &   Frequency elements \\ 
$\boldsymbol{v}$ & Unit vector\\
$\mathbb{B}_d$ &  Unit Euclidean ball \\
$R_d\mathbb{B}_d$ & Unit-volume Euclidean ball, $R_d$ is a selected constant.\\
$\boldsymbol{1}_{\boldsymbol{\omega}\in R_d\mathbb{B}_d}$ & Indicator function that elements within $R_d\mathbb{B}_d$ are one\\
$\varphi(\boldsymbol{x})$ & Fourier transform of $\boldsymbol{1}_{\boldsymbol{\omega}\in R_d\mathbb{B}_d}$\\
$\mu=\varphi^2(\boldsymbol{x})$ & A general measure\\
$\mathbb{E}_{\x \sim \mu} (p-q)^2$ & The distance between $p$ and $q$ under the measure $\mu$\\
$f_1$                            & A one-hidden-layer conventional network                  \\ 
$f_2$                            & A one-hidden-layer quadratic network                     \\ 
$f_3$                            & A one-hidden-layer heterogeneous network                 \\ 
$\tilde{g}_{ip}(\boldsymbol{x})$ & The constructed inner-product function                                   \\ 
$\tilde{g}_{r}(\boldsymbol{x})$  & The constructed radial function                                          \\ 
$c_\sigma$                       & The constant dependent on $\sigma$              \\ 
$c_1, c_2$                       & Constants associated with the quadratic network           \\ 
$\epsilon_1$                     & Approximation error associated with quadratic networks    \\ 
$c_3, c_4$                       & Constants associated with the conventional network              \\ 
$\epsilon_2$                     & Approximation error associated with conventional networks \\ 
$C_1, C_2$                       & Constants associated with heterogeneous networks \\
$\operatorname{Span}\{\v\}$ & The set of all linear combinations of $\v$ \\
$\operatorname{Span}\{\v\} +R_{d}\mathbb{B}_{d}$ & The region resultant by convoluting $\operatorname{Span}\{\v\}$ and $R_{d}\mathbb{B}_{d}$ \\
$m_i(\x)$ & A continuous radial function\\
$\mathbb{S}^{d-1}$  &The unit Euclidean sphere in $\mathbb{R}^d$\\
$T$   &  $T = \operatorname{Span}\left\{{\boldsymbol{v}}\right\}+R_{d} \mathbb{B}_{d}$ \\
$r$ & The radius \\
$h(r)$ & $\frac{Area\left(r \mathbb{S}^{d-1} \cap T\right)}{Area\left(r \mathbb{S}^{d-1}\right)}$ \\
$c$ & A constant that restricts the upper bound of $h(r)$\\
$l$ & The function $l=\frac{\widehat{\tilde{g}_{ip}\varphi }}{\left\| \tilde{g}_{ip}\varphi \right\| _{L_2}}$\\
$q$ & The function $q=\frac{\widehat{f_2\varphi }}{\left\| f_2\varphi \right\| _{L_2}}$ \\
$\tilde{g}_{ip}^{(t)}(\x)$ & Truncating $\tilde{g}_{ip}(\x)$ 
\\ \bottomrule
\end{tabular}}}
\label{tab:symbols}
\end{table}

\textbf{Preliminaries and assumptions. } {Symbols used in our theorem are shown in Table \ref{tab:symbols}. Given arbitrary functions $p$ and $q$, we have the following assumptions.} i) $\|\cdot\|$ is the $L_2$ norm, \textit{i.e.}, $\|p\|^{2}=\int_{\boldsymbol{x}}{p^2(\boldsymbol{x})d\boldsymbol{x}}$. ii) The inner product of two functions is $\langle p,q\rangle_{L_2}=\int_{\boldsymbol{x}}{p(\boldsymbol{x})q(\boldsymbol{x})d\boldsymbol{x}}$. iii) $\varphi(\boldsymbol{x})=(\frac{R_d}{\|\boldsymbol{x}\|})^{d/2}J_{d/2}(2\pi R_d\|\boldsymbol{x}\|)$, where $J_{d/2}: \mathbb{R} \to \mathbb{R}$ is the Bessel function of the first kind. $\varphi(\boldsymbol{x})$ is the Fourier transform of the unit-volume ball in the frequency domain $\boldsymbol{1}_{\boldsymbol{\omega}\in R_d\mathbb{B}_d}$. iv) For simplicity, we use $\hat{p}(\boldsymbol{\omega})$ to denote the Fourier transform of $p(\x)$. v) {The $L_2$-norm of $p(\x)$ under a general measure $\mu=\varphi^2(\x)$ is $\left[\int_{\boldsymbol{x}}{p^2(\boldsymbol{x}) \varphi^2(\boldsymbol{x}) d\boldsymbol{x}}\right]^{1/2}$.} The distance between $p$ and $q$ under {a general measure $\mu=\varphi^2(\x)$ equals to}

\begin{equation}
    \int (p-q)^2\varphi^2 d\x = \Vert p\varphi-q\varphi \Vert^2 = \Vert \hat{p}*\textbf{1}_{\boldsymbol{\omega}\in R_d\mathbb{B}_d}-\hat{q}*\textbf{1}_{\boldsymbol{\omega}\in R_d\mathbb{B}_d} \Vert^2,
\end{equation}
where $*$ is a convolution. For simplicity, let $\mathbb{E}_{\x \sim \mu} (p-q)^2=\int (p-q)^2\varphi^2 d\x$. vi) A one-hidden-layer conventional network of $k$ neurons is $f_1(\x) = \sum_{i=1}^k a_i\sigma(b_i {\x^\top} \v_i+c_i)$, while a one-hidden-layer quadratic network is simplified as $f_2(\x) = \sum_{i=1}^k a_i\sigma(b_i\x^\top\x +c_i)$, in which keeping only the power term is sufficient to express the construction. vii) All
theorems assume that the activation function $\sigma(\cdot)$ satisfies
$|\sigma(x)| \leq C(1 + |x|^\alpha), x \in \mathbb{R}$ for some constants {$C$} and $\alpha$, {where $C$ and $\alpha$ are dependent on $\sigma(\cdot)$}. viii) {$\operatorname{Span}\{\v\}$ is the set of all linear combinations of $\v$.} ix) {$\operatorname{Span}\{\v\} +R_{d}\mathbb{B}_{d}$ is the region resultant from convolution between two regions: $\operatorname{Span}\{\v\}$ and $R_{d}\mathbb{B}_{d}$. Geometrically, $\operatorname{Span}\{\v\} +R_{d}\mathbb{B}_{d}$ is a hyper-cylinder. } As Figure \ref{fig:frequency} shows, Fourier transform of
a radial function is also radial in the frequency domain, while Fourier transform of an inner-product based function {$g(\x^\top\v)$} is supported over a line, denoted as {$\operatorname{Span}\{\v\}$}. Then, $\hat{g}*\textbf{1}_{\boldsymbol{\omega}\in R_d\mathbb{B}_d}$ will be supported over a region resultant from convolution between two regions: $\operatorname{Span}\{\v\}$ and $R_{d}\mathbb{B}_{d}$. Geometrically, $\operatorname{Span}\{\v\} +R_{d}\mathbb{B}_{d}$ is a hyper-cylinder.

\begin{figure}[htb]
\center{\includegraphics{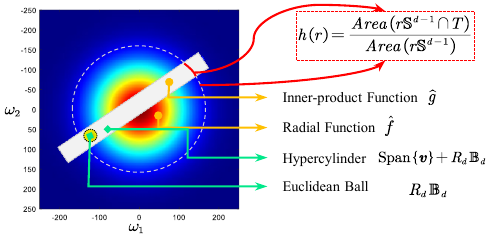}}
\caption{Fourier transforms of a radial function $f$ and an inner-product based function $g$ in the frequency domain.} 
\label{fig:frequency}
\end{figure}


\begin{theorem}[Main]
There exist universal constants {$c_\sigma$,} $c_1,c_2,c_3,c_4, C_1, C_2, \epsilon_1, \epsilon_2>0$ {such that the following holds:} For every dimension $d$, there exists a {measure $\mu$} and a function $ \tilde{g}_{ip}(\boldsymbol{x})+ \tilde{g}_{r}(\boldsymbol{x}): \mathbb{R}^d \to \mathbb{R}$ with the following properties:
\begin{enumerate}[label=\theenumi.]

    \item Every $f_1$ expressed by a one-hidden-layer conventional network of width at most $\frac{1}{2}c_1e^{c_2d}$ has 
    \begin{equation}
        \mathbb{E}_{\x \sim \mu} [f_1-(\tilde{g}_{ip}+\tilde{g}_{r})]^2 \geq \epsilon_2.
    \end{equation}

    \item Every $f_2$ expressed by a one-hidden-layer quadratic network of width at most $\frac{1}{2}c_3e^{c_4d}$ has
    \begin{equation}
        \mathbb{E}_{\x \sim \mu} [f_2-(\tilde{g}_{ip}+\tilde{g}_{r})]^2 \geq \epsilon_1.
    \end{equation}

        \item There exists a function $ f_3$ expressed by a one-hidden-layer heterogeneous network that can approximate $\tilde{g}_{ip}+\tilde{g}_{r}$ with $C_1c_\sigma d^{3.75}+C_2c_\sigma$ neurons.
\end{enumerate}
\label{thm:main}
\end{theorem}

\textbf{The sketch of proof} (Table \ref{results:sketch}). It has been shown that there exists a function $\tilde{g}_{r}$ hard to approximate by a one-hidden-layer conventional network and easy to approximate by a one-hidden-layer quadratic network (Theorem \ref{thm:quadratic}). We only need to construct a function $\tilde{g}_{ip}$ hard to approximate by a one-hidden-layer quadratic network and easy to approximate by a one-hidden-layer conventional network (Theorem \ref{thm:core}). Then, $\tilde{g}_{ip}+\tilde{g}_{r}$ will be hard for both conventional and quadratic networks but easy for a heterogeneous network (Theorem \ref{thm:main}).

{\textbf{Remark 1}. The standard quadratic neuron is 
$
\sigma((\boldsymbol{x}^\top\boldsymbol{w}^{r}+b^{r})(\boldsymbol{x}^\top\boldsymbol{w}^{g}+b^{g})+(\boldsymbol{x}\odot\boldsymbol{x})^\top\boldsymbol{w}^{b}+c)
$. By excluding interaction terms, we do not intend to weaken the expressive ability of quadratic neurons to establish the superiority of heterogeneous networks. Instead, this is because the target function $\tilde{g}_{ip}$ is a function made of an inner product. Our theorem is existential. Even if we include interaction terms, to represent $\tilde{g}_{ip}$, we can find a network whose every quadratic neuron actually takes $\boldsymbol{w}^{g}=0, b^g=1, \boldsymbol{w}^{b}=0, c=0$. Such a network is essentially a conventional network, with quadratic neurons degenerating into conventional ones. In this regard, it is true that a standard quadratic network can express $\tilde{g}_{ip}+\tilde{g}_r$. But because some quadratic neurons degenerate into conventional neurons, this quadratic network is essentially a heterogeneous network.
If the constructed function $\tilde{g}_{ip}$ is not based on inner-product, the original form of quadratic neurons must be adopted. }


\begin{table}[htbp]
  \centering
  \caption{The approximability of $\tilde{g}_{ip}$, $\tilde{g}_{r}$, and $\tilde{g}_{ip}+\tilde{g}_{r}$.}
  \scalebox{0.9}{
  \begin{tabular}{ccccc}
  \toprule
      & construction & \makecell{conventional \\ $f_1$} & \makecell{quadratic \\ $f_2$} & \makecell{heterogeneous \\ $f_1+f_2$} \\
     \midrule
    Theorem \ref{thm:quadratic} &  $\tilde{g}_{r}$ (radial) &  \xmark & \cmark & -\\
    Theorem \ref{thm:core} & $\tilde{g}_{ip}$ (inner-product) & \cmark & \xmark & -\\
    {Theorem \ref{thm:main}} & {$\tilde{g}_{ip}+\tilde{g}_{r}$} &  {\xmark} & {\xmark} & {\cmark} \\
    \bottomrule
    \end{tabular}}%
    \label{results:sketch}
    \vspace{-0.3cm}
\end{table}%

\begin{theorem}
There exist universal constants {$c_\sigma$,} $c_1, c_2, C_1, {\epsilon_1}>0$ {such that the following holds:} For every dimension $d$, there exists {a measure $\mu$} and an inner-product based function $ \tilde{g}_{ip}(\boldsymbol{x}): \mathbb{R}^d \to \mathbb{R}$ with the following properties:
\begin{enumerate}[label=\theenumi.]

    \item Every $f_2$ that is expressed by a one-hidden-layer quadratic network of width at most $\frac{1}{2}c_1e^{c_2d}$ has
    \begin{equation}
        \mathbb{E}_{\x \sim \mu} [f_2-\tilde{g}_{ip}]^2 \geq {\epsilon_1}.
    \end{equation}
   
    \item There exists a function $f_1$ that is expressed by a one-hidden-layer conventional network that can approximate $\tilde{g}_{ip}$ with $C_1c_\sigma$ neurons.
\end{enumerate}
\label{thm:core}
\end{theorem}

{Based on our proof, the relation between $\epsilon_1$ and $c_1,c_2$ is characterized by $\epsilon_1 \leq \|\tilde{g}_{ip}\|_{L_{2}(\mu)}  \sqrt{2 (c_{1} / 2-k \exp \left(-c_{2} d\right)}$, while $0<\epsilon_1 \leq \|\tilde{g}_{r}\|_{L_{2}(\mu)}  \sqrt{2(c_{3} / 2-k \exp \left(-c_{4} d\right)}$ according to Theorem 1 of \cite{fan2020universal}.}

\begin{theorem}[Theorem 1 of \cite{fan2020universal}]
There exist universal constants $c_\sigma, c_3, c_4, C_2, \epsilon_2>0$ {such that the following holds:} For every dimension $d$, there exists {a measure $\mu$} and a radial function $ \tilde{g}_{r}(\boldsymbol{x}): \mathbb{R}^d \to \mathbb{R}$ with the following properties:
\begin{enumerate}[label=\theenumi.]

    \item Every $f_1$ that is expressed by a one-hidden-layer conventional network of width at most {$\frac{1}{2}c_3e^{c_4d}$} has
    \begin{equation}
        \mathbb{E}_{\x \sim \mu} [f_1-\tilde{g}_{r}]^2 \geq {\epsilon_2}.
    \end{equation}
    
        \item There exists a function $f_2$ expressed by a one-hidden-layer quadratic network that can approximate $\tilde{g}_{r}$ with {$C_2c_\sigma d^{3.75}$} neurons.
\end{enumerate}
\label{thm:quadratic}
\end{theorem}

\begin{proof}
The detailed proof can be found in Theorem 1 of \cite{fan2020universal} and Proposition 13 of \cite{eldan2016power}. 
\end{proof}

\textbf{The key idea of constructing $\tilde{g}_{ip}$.} The purpose of construction is to enlarge the difference between $\tilde{g}_{ip}$ and $f_2$. First, since $\hat{f_2}(\boldsymbol{\omega})$ is radial, $\hat{\tilde{g}}_{ip}(\boldsymbol{\omega})$ is preferred to be supported over as small area as possible such that $\hat{f_2}$ and $\hat{\tilde{g}}_{ip}$ have a small overlap. Therefore, $\tilde{g}_{ip}$ is cast as an inner-product based function whose Fourier transform is only supported over a line. Second, still because $\hat{f_2}(\boldsymbol{\omega})$ is radial, we wish $\tilde{g}_{ip}$ to have unneglectable high-frequency elements because the portion of overlap between $f_2$ and $\tilde{g}_{ip}$ is lower in high-frequency domain than in low-frequency domain.
Specifically, given a unit vector $\boldsymbol{v}$, let $\tilde{g}_{ip}(\x)=\frac{1}{2\pi}\sqrt{\frac{4R_d}{1-\delta}}\mathrm{sinc}\Big(\frac{2R_d}{1-\delta} {\x^\top\boldsymbol{v}}\Big), \delta \in [0,1]$. In the following, Lemma \ref{lem:g_2construction} implies that $\tilde{g}_{ip}$ has certain high-frequency elements, and Lemma \ref{lemma:gap} suggests that an inner-product based function and a radial function is incompatible, since their inner-product is upper bounded exponentially.

\begin{lemma}
Given a unit vector $\v$, $\tilde{g}_{ip}(\x)=\frac{1}{2\pi}\sqrt{\frac{4R_d}{1-\delta}}\mathrm{sinc}\Big(\frac{2R_d}{1-\delta} {\x^\top\boldsymbol{v}}\Big)$ satisfies 
$\int_{2R_d\mathbb{B}_{d}} \hat{\tilde{g}}_{ip}^2(\boldsymbol{\omega}) d \boldsymbol{\omega} = 1-\delta, \delta \in [0,1]$.
\label{lem:g_2construction}
\end{lemma}

\begin{proof}
Denote the support of $\hat{\tilde{g}}_{ip}(\boldsymbol{\omega})$ as $\operatorname{Span}\{\v\}$, which is $\{\boldsymbol{\omega}=t\v ~|~t \in \mathbb{R}\}$. $\tilde{g}_{ip}(\x)$ is mathematically formulated as
\begin{equation}
   \hat{\tilde{g}}_{ip}(\boldsymbol{\omega})=
   \begin{cases}
   & \sqrt{\frac{1-\delta}{4R_d}}, \ \ \ \boldsymbol{\omega}=t\v, t \in [-\frac{2R_d}{1-\delta},\frac{2R_d}{1-\delta}] \\
   & \ \ \ \ \ 0,  \ \ \ \ \ \boldsymbol{\omega}=t\v, t \notin [-\frac{2R_d}{1-\delta},\frac{2R_d}{1-\delta}].
   \end{cases}
\end{equation}
Because $\hat{\tilde{g}}_{ip}$ is a constant over $[-\frac{2R_d}{1-\delta},\frac{2R_d}{1-\delta}]$, we have $\int_{2R_d\mathbb{B}_{d}} \hat{\tilde{g}}_{ip}^2(\boldsymbol{\omega}) d \boldsymbol{\omega} = \int_{2R_d\mathbb{B}_{d}} \frac{1-\delta}{4R_d} d \boldsymbol{\omega} = \frac{1-\delta}{4R_d} \int_{[-2R_d, 2R_d]} dt = 1-\delta$, which concludes the proof.
\end{proof}

\begin{lemma}
Let $g, f: \mathbb{R}^d \to \mathbb{R}$ be two functions of unit norm. Moreover, $\int_{2R_d\mathbb{B}_d}{g^2(\boldsymbol{x})d\boldsymbol{x}} \leq 1 - \delta$, $\delta \in [0,1]$, $g$ is supported over 
$\operatorname{Span}\{\v\} +R_{d}\mathbb{B}_{d}$, and $f$ is a combination of $k$ radial functions, denoted as $f=\sum_{i=1}^k m_i( \Vert\x\Vert )$, {where $m_i(\Vert\x\Vert )$ is a continuous radial function.} Then,
$$
\left \langle f,g \right \rangle_{L_{2}} \leq 1-\frac{\delta }{2} + k\exp(-cd/2) ,
$$
{where $c>0$ is a constant dependent on $f$.}
\label{lemma:gap}
\end{lemma}

\begin{proof}
Let $T = \operatorname{Span}\left\{{\boldsymbol{v}}\right\}+R_{d} \mathbb{B}_{d}$.
For any radius $r>0$, define
$h(r)=\frac{Area\left(r \mathbb{S}^{d-1} \cap T\right)}{Area\left(r \mathbb{S}^{d-1}\right)}$,
where $\mathbb{S}^{d-1}$ is the unit Euclidean sphere in $\mathbb{R}^d$, and $Area(\cdot)$ is to compute the area. The geometric meaning of $h(r)$ is the ratio of the area of intersection of $T$ and $r\mathbb{S}^{d-1}$ vs the area of $r\mathbb{S}^{d-1}$. Because $T$ is a hyper-cylinder, following the illustration in page 7 of \cite{eldan2016power}, $h(r)$ is exponentially small, \textit{i.e.,} $\exists \ c>0$, such that
\begin{equation}
    h(2R_d) \leq \exp(-cd).
    \label{eq:proof_ratio}
\end{equation}
As $r$ increases, the area of the intersection decreases but the area of the sphere increases. Therefore, $h(r)$ is monotonically decreasing.  

Our goal is to show $\left \langle f,g \right \rangle_{L_{2}}$ can be upper bounded. We decompose the inner product $\left \langle f,g \right \rangle_{L_{2}}$ into the integrals in a hyperball $2R_d\mathbb{B}_d$ and {the region out of the hyperball $(2R_d\mathbb{B}_d)^C$}. Mathematically,  
\begin{equation}
\left \langle f,g \right \rangle_{L_{2}} = \int_{2R_d\mathbb{B}_d }f(\boldsymbol{x})g(\boldsymbol{x})d\boldsymbol{x} + \int_{(2R_d\mathbb{B}_d)^C }f(\boldsymbol{x})g(\boldsymbol{x})d\boldsymbol{x} .
\label{eqn:decomposition}
\end{equation}

Now, we calculate the above two integrals, respectively. For the first integral, we bound it as follows:
\begin{equation}
    \begin{aligned}
     \int_{2R_d\mathbb{B}_d}{f\left( \boldsymbol{x} \right) g\left( \boldsymbol{x} \right) d\boldsymbol{x}} 
    & \overset{\left( 1 \right)}{\leq}  \left\| g\cdot \boldsymbol{1}_{\{ {2R_d\mathbb{B}_d}\}}  \right\| _{{L_2}}\left\| f \right\| _{{L_2}} \\
    &\overset{\left( 2 \right)}\leq \sqrt{1-\delta},
    \end{aligned}
    \label{eqn:part1}
\end{equation}
(1) follows from the Cauchy-Schwartz inequality; (2) follows from the fact that $f(\cdot)$ has unit $L_2$ norm and {$\int_{2R_d\mathbb{B}_d}{g^2(\boldsymbol{x})d\boldsymbol{x}} \leq 1 - \delta$}. 

For the second integral, we have 
$\int_{(2R_d\mathbb{B}_d)^C }{f\left( \boldsymbol{x} \right) g\left( \boldsymbol{x} \right) d\boldsymbol{x}} = \int_{(2R_d\mathbb{B}_d)^C \bigcap{T} }{f\left( \boldsymbol{x} \right) g\left( \boldsymbol{x} \right) d\boldsymbol{x}}$ because $g(\x)=0$ when $\x \notin T$. Then, we have
\begin{equation}
\small
\begin{aligned}
&\int_{(2R_d\mathbb{B}_d)^C \bigcap{T} }{f\left( \boldsymbol{x} \right) g\left( \boldsymbol{x} \right) d\boldsymbol{x}}  \\
=& \sum_{i=1}^k \int_{(2R_d\mathbb{B}_d)^C \bigcap{T} }{m_i\left( \Vert \boldsymbol{x} \Vert \right) g\left( \boldsymbol{x} \right) d\boldsymbol{x}} 
\\
\leq & \sum_{i=1}^k \Bigg( \sqrt{\int_{(2R_d\mathbb{B}_d)^C \bigcap{T}}{m_i^2\left( \Vert \boldsymbol{x} \Vert \right) d\boldsymbol{x}}}\sqrt{\int_{(2R_d\mathbb{B}_d)^C \bigcap{T}}{g^2\left( \boldsymbol{x} \right) d\boldsymbol{x}}} \Bigg)
\\
\overset{\left( 1 \right)}{\leq} & \sum_{i=1}^k \sqrt{\int_{r\geq 2R_d}{\int_{\left( r\mathbb{S}^{d-1} \right) \bigcap{T}}{m_i^2\left( r \right) d\mathbb{S}dr}}}
\\
= & \sum_{i=1}^k\sqrt{\int_{r\geq 2R_d}{\left( \int_{r\mathbb{S}^{d-1}}{m_i^2\left( r \right) d\mathbb{S}} \right) \cdot \Big[\frac{\int_{\left( r\mathbb{S}^{d-1} \right) \cap T}{m_i^2\left( r \right) d\mathbb{S}}}{\int_{r\mathbb{S}^{d-1}}{m_i^2\left( r \right)  d\mathbb{S}}} \Big]dr}}
\\
\overset{\left( 2 \right)}{=}  & \sum_{i=1}^k \sqrt{\int_{r\geq 2R_d}{\int_{r\mathbb{S}^{d-1}}{m_i^2\left( r \right) d\mathbb{S} }\cdot h\left( r \right) dr}}
\\
\overset{\left( 3 \right)}{\leq} & \sum_{i=1}^k  \sqrt{h\left( 2R_d \right) \int_{r\geq 2R_d}{\int_{r\mathbb{S}^{d-1}}{m_i^2\left( r \right) d\mathbb{S}}dr}} 
\\
\overset{(4)}{=} & \sum_{i=1}^k \sqrt{ h\left( 2R_d \right)} \overset{(5)} \leq k\exp \left( -cd/2 \right). 
\end{aligned}
\label{eqn:part2}
\end{equation}
In the above, {(1) follows from the fact that $g(\cdot)$ has a unit $L_2$ norm}; (2) holds because $m_i(r)$ is radial; (3) follows from $h(r)$ is monotonically decreasing; (4) follows from $f(\cdot)$ has a unit norm; {(5) follows from Eq. (\ref{eq:proof_ratio})}.

Combining Eqs. \eqref{eqn:decomposition}, \eqref{eqn:part1} and \eqref{eqn:part2}, we have $\left \langle f,g \right \rangle_{L_{2}} \leq \sqrt{1-\delta} + k\exp(-cd/2) \leq 1 - \frac{\delta}{2} + k\exp(-cd/2).$

\end{proof}


\textit{Proof of \textbf{Theorem} \ref{thm:core}. } The proof of Theorem \ref{thm:core} consists of two parts: (i) the inapproximability of a quadratic network to $\tilde{g}_{ip}$; (ii) the approximability of a conventional network to $\tilde{g}_{ip}$.

(i) Define $l=\frac{\widehat{\tilde{g}_{ip}\varphi }}{\left\| \tilde{g}_{ip}\varphi \right\| _{L_2}}$. Because for $\tilde{g}_{ip}$, we have $\int_{2R_d\mathbb{B}_d}{\tilde{g}_{ip}^2(\boldsymbol{x})d\boldsymbol{x}} = 1 - \delta$,
there must exist a universal constant $c_1 \in [0,1]$ such that $\int_{2R_d\mathbb{B}_d}{l^2(\boldsymbol{x})d\boldsymbol{x}} \leq 1-c_{1}$. Define the function $q=\frac{\widehat{f_2\varphi }}{\left\| f_2\varphi \right\| _{L_2}}$, where {$f_2= \sum_i^k a_i\sigma(\omega_i\x^\top \x+b_i)$ } is a radial function expressed by a quadratic network. Thus, according to Lemma \ref{lemma:gap}, the functions $l{(\cdot)}$, $q{(\cdot)}$ satisfy 
\begin{equation}
    \left \langle l{(\cdot)},q{(\cdot)} \right \rangle_{L_{2}} \leq 1-\frac{c_1}{2} + k\exp(-c_{2}d/2) ,
\label{eq:proof2_qwinnerproduct}
\end{equation}
with $c_2>0$ being a universal constant. 

For every scalars $\beta_1, \beta_2 > 0$, we have
\begin{equation}
       \left\| \beta _1l{(\cdot)}-\beta _2q{(\cdot)} \right\| _{L_2}\geqslant \frac{\beta _2}{2}\left\| l{(\cdot)}-q{(\cdot)} \right\| _{L_2}. 
\label{eqn:gap_with_coeff}       
\end{equation}
The reason why it holds true is as follows:

Without loss of generality, we assume that $\beta_2=1$. For two unit vectors $u$,
$v$ in one Hilbert space, it has 
\begin{equation}
\nonumber
\begin{aligned}
& \min _{\beta}\|\beta v-u\|^{2} =\min _{\beta}\left(\beta^{2}\|v\|^{2}-2 \beta\langle v, u\rangle+\|u\|^{2}\right) \\
=& \min _{\beta}\left(\beta^{2}-2 \beta\langle v, u\rangle+1\right) =1-\langle v, u\rangle^{2}=\frac{1}{2}\|v-u\|^{2}.
\end{aligned}
\end{equation}

Next, combining Eqs. \eqref{eq:proof2_qwinnerproduct} and \eqref{eqn:gap_with_coeff}, and utilizing that $q, l$ have unit $L_2$ norm, we have
\begin{equation}
\begin{aligned}
&\|f_2-\tilde{g}_{ip}\|_{L_{2}(\mu)}=\|f_2 \varphi-\tilde{g}_{ip} \varphi\|=\|\widehat{f_2 \varphi}-\widehat{\tilde{g}_{ip} \varphi}\| \\
=& \left\|\left(\|f_2 \varphi\|\right) q(\cdot)-\left(\|\tilde{g}_{ip} \varphi\|_{L_{2}}\right) l(\cdot)\right\|\\
\geq & \frac{1}{2}\|\tilde{g}_{ip} \varphi\|\|q{(\cdot)}-l{(\cdot)}\|=\frac{1}{2}\|\tilde{g}_{ip}\|_{L_{2}(\mu)}\|q{(\cdot)}-l{(\cdot)}\|_{L_{2}}\\
\geq & \frac{1}{2} {\|\tilde{g}_{ip}\|_{L_{2}(\mu)}}  \sqrt{2\left(1-\langle q, l\rangle_{L_{2}}\right)} \\
\geq &  {\|\tilde{g}_{ip}\|_{L_{2}(\mu)}}  \sqrt{2 \max \left(c_{1} / 2-k \exp \left(-c_{2} d\right), 0\right)},
\end{aligned}
\end{equation}
where the first inequality is due to Eq. \eqref{eqn:gap_with_coeff}. Therefore, as long as $k\leq \frac{c_1}{2}e^{c_2d}$, $\|f_2-\tilde{g}_{ip}\|_{L_{2}(\mu)}$ is larger than a positive constant ${\epsilon_1}$.  

(ii) In contrast, $f_1$ can approximate $\tilde{g}_{ip}(\x)=\frac{1}{2\pi}\sqrt{\frac{4R_d}{1-\delta}}\mathrm{sinc}\Big(\frac{2R_d}{1-\delta} {\x^\top\boldsymbol{v}}\Big)$ with a polynomial number of neurons. From the proof of Theorem 1 in \cite{debao1993degree}, $\forall {H}: \mathbb{R}\to \mathbb{R}$ which is constant outside a bounded interval $[r_1,r_2]$, there exist scalars $a, \{\alpha_i, \beta_i, \gamma_i\}_{i=1}^k$, where $k \leq c_\sigma \frac{(r_2-r_1)L}{\epsilon}$ such that the function 
$h(t)=a+\sum_{i=1}^k \alpha_i \cdot \sigma(\beta_i t-\gamma_i)$
satisfies that $\underset{t \in \mathbb{R}}{\sup} \ \  |{H}(t)-h(t)|\leq \epsilon$.

According to this theorem, we define an auxiliary function $\tilde{g}_{ip}^{(t)}(\x)$ by truncating $\tilde{g}_{ip}(\x)$ when ${\x^\top \v} \notin [-\frac{1-\delta}{R_d\epsilon}, \frac{1-\delta}{R_d\epsilon}]$ to be zero, then $  |\tilde{g}_{ip}^{(t)}(\x)-\tilde{g}_{ip}(\x)| \leq |\tilde{g}_{ip}(\x)| \leq \frac{1}{2\pi} \sqrt{\frac{4R_d}{1-\delta}}/(\frac{2R_d}{1-\delta}\cdot {\x^\top \v}) \leq \epsilon/2$ when ${\x^\top \v} \notin [-\frac{\sqrt{1-\delta}}{\pi \sqrt{R_d}\epsilon}, \frac{\sqrt{1-\delta}}{\pi \sqrt{R_d}\epsilon}]$. Applying the aforementioned theorem to $\tilde{g}_{ip}^{(t)}(\x)$, we have $f_1(\x)=a+\sum_{i=1}^k \alpha_i \cdot \sigma(\beta_i {\x^\top \v}-\gamma_i)$ that can approximate $\tilde{g}_{ip}^{(t)}$ in terms of $\underset{\x \in \mathbb{R}^d}{\sup} \ \  |\tilde{g}_{ip}^{(t)}(\x)-f_1(\x)|\leq \epsilon/2$.
Here, $r_1=-\frac{\sqrt{1-\delta}}{\pi\sqrt{R_d}\epsilon}$ and $r_2=\frac{\sqrt{1-\delta}}{\pi\sqrt{R_d}\epsilon}$, the number of neurons needed is no more than $c_\sigma \frac{2\sqrt{1-\delta}L}{\pi\sqrt{R_d}\epsilon} \leq c_\sigma \frac{2\times \sqrt{5.264} \sqrt{1-\delta}L}{\pi\epsilon} = C_1 c_\sigma$. Because the geometric meaning of $1/R_d$ is the volume of $d$-hyperball, which is upper bounded by $(1/R_d)_{d=5}\approx 5.264$\footnote{\url{https://en.wikipedia.org/wiki/Volume_of_an_n-ball}}, the number of needed neurons is irrelevant to $d$. Furthermore, $f_1$ fulfills
\begin{equation}
\begin{aligned}
  &\underset{\x \in \mathbb{R}^d}{\sup} |\tilde{g}_{ip}(\x)-f_1(\x)| \\
 \leq &  \underset{\x \in \mathbb{R}^d}{\sup} |\tilde{g}_{ip}(\x)-\tilde{g}_{ip}^{(t)}(\x)| + \underset{\x \in \mathbb{R}^d}{\sup} |\tilde{g}_{ip}^{(t)}(\x)-f_1(\x)|
 \leq \epsilon.
\end{aligned}
\end{equation}
$\hfill\square$

\textit{Proof of \textbf{Theorem} \ref{thm:main}. }
Theorems \ref{thm:quadratic} and \ref{thm:core} suggest that a conventional network $f_1$ can express $\tilde{g}_{ip}$ with a polynomial number of neurons, and a quadratic network $f_2$ can express $\tilde{g}_{r}$ with a polynomial number of neurons. Let a heterogeneous network be $f_3=f_1+f_2$, $f_3$ can straightforwardly express $\tilde{g}_{ip}+\tilde{g}_{r}$ with a polynomial number of neurons. 

Next, we need to show how to deduce from $\mathbb{E}_{\x \sim \mu} [f_2-\tilde{g}_{ip}]^2 \geq {\epsilon_1}$ to $\mathbb{E}_{\x \sim \mu} [f_2-(\tilde{g}_{ip}+\tilde{g}_{r})]^2 \geq \epsilon_1$. Here, we slightly abuse $\epsilon_1$ for succinctness. Both formulas mean that there exists a gap between functions. We can rewrite $\tilde{g}_{r}$ as 
\begin{equation}
    \tilde{g}_{r}(\|\x\|)= \tilde{g}_{r}\circ \sqrt{\|\x\|^2}.
\end{equation}
Thus, we can express $\tilde{g}_{r}$ by a quadratic neuron using a special activation function {$\sigma'=\tilde{g}_{r}\circ \sqrt{\cdot}$}. It holds that {$\sigma'(x) \leq C'(1+|x|^{\alpha'})$} {for certain constants $C'$ and $\alpha'$}, since both $\tilde{g}_{r}$ and $s$ are bounded. As a result, regarding $f_2-\tilde{g}_{r}$ as a quadratic network, we can get $ \mathbb{E}_{\x \sim \mu} [f_2-(\tilde{g}_{ip}+\tilde{g}_{r})]^2 =\mathbb{E}_{\x \sim \mu} [(f_2-\tilde{g}_{r})-\tilde{g}_{ip}]^2\geq \epsilon_1$.

Similarly, $\tilde{g}_{ip}$ can be re-expressed by a conventional neuron with a bounded activation. We can also get $\mathbb{E}_{\x \sim \mu} [f_1-(\tilde{g}_{ip}+\tilde{g}_{r})]^2 \geq \epsilon_2$ from $\mathbb{E}_{\x \sim \mu} [f_1-\tilde{g}_{r}]^2 \geq {\epsilon_2}$.

$\hfill\square$. 

\textbf{Remark 2}. Will \textbf{Theorem \ref{thm:main}} change if we use more layers? {Indeed, our theorem and the associated proof cannot be extended into deeper networks. As Table \ref{results:sketch} shows, the key ingredient of our main theorem is constructing $\tilde{g}_{ip}+\tilde{g}_{r}$, wherein $\tilde{g}_{ip}$ can be approximated by a one-hidden-layer conventional network with a polynomial number of neurons but a one-hidden-layer quadratic network needs an exponential number of neurons, and $\tilde{g}_{r}$ can be approximated by a one-hidden-layer quadratic network with a polynomial number of neurons but a one-hidden-layer conventional network needs an exponential number of neurons. The main reason accounting for such exponential differences is that it is hard for a one-hidden-layer conventional network to represent a radial function and a one-hidden-layer quadratic network to represent an inner-product-based function. However, if one allows a network to have more hidden layers, a conventional network can use a polynomial number of neurons to represent the function $g(x)=x^2$ first, and then use $x^2$ to construct a radial function with a polynomial number of neurons. Thus, the total number of neurons is still polynomial. In this light, although heterogeneous networks still enjoy the advantage of efficiency in representing $\tilde{g}_{ip}+\tilde{g}_{r}$, the difference is not exponential. A similar analysis also holds for $\tilde{g}_{ip}$ and the quadratic.} We leave the extension to deeper networks as our future work. 

\fl{Despite assuming a simple structure, our main theorem holds significant implications for unsupervised anomaly detection. As per the manifold hypothesis, anomalies often exhibit evident abnormal features in a low-dimensional space but become hidden or unnoticeable in a high-dimensional space~\cite{fefferman2016testing}. Autoencoders are employed to extract the latent features of the data, thereby amplifying the differences between normal instances and anomalies~\cite{ruff2021unifying}. When dealing with high-dimensional and complicated data, particularly in the context of unsupervised learning, our theorem assures that heterogeneous networks can approximate more flexible, more expressive, and nonlinear decision functions, thereby demonstrating superior representational ability compared to conventional networks. Furthermore, in our primary experiments, encoders and decoders for anomaly detection are designed to consist of a single hidden layer. Such heterogeneous networks indeed achieve competitive performance, as suggested by our theory, thereby showcasing high efficiency in the design of anomaly detection autoencoders.}

\section{Heterogeneous Autoencoder}
Given the input $\boldsymbol{x} \in \mathbb{R}^n$, let $E(\cdot)$ be the encoder and $D(\cdot)$ be the decoder, an autoencoder attempts to learn a mapping from the
input to itself by optimizing the following loss function: 
$\underset{\boldsymbol{W}}{\min}~\mathcal{J}(\boldsymbol{W};\boldsymbol{x})=\Vert D(E(\boldsymbol{x}))-\boldsymbol{x} \Vert^2$,
where $W$ is a collection of all network parameters for simplicity. $\boldsymbol{z}= E(\boldsymbol{x})$ is the learned embedding in the
low-dimensional space. {The objective of the autoencoder is to learn $\boldsymbol{z}$ such that the output layer can accurately reproduce the original input. Within this structure, we refer to autoencoders that only use conventional neurons or quadratic neurons in the encoder and decoder as homogeneous autoencoders, which are referred to as AE and QAE, respectively.}

{Our theoretical derivation suggests combining the conventional and quadratic neurons in a network is a promising direction to explore. But it does not provide specific guidelines on which structure is the best for HAEs. Thus, we can only heuristically design the structures of the autoencoder based on a symmetry consideration: 1) utilizing equal depth and width in the encoder and decoder, which fulfills the standard practice of autoencoders; 2) arranging conventional and quadratic layers in a symmetric fashion to put them on an equal booting, since we don't have any prior knowledge which neuron type is more important. Symmetric consideration can naturally result in three heterogeneous autoencoder designs}, referred to as HAE-X, HAE-Y, and HAE-I, respectively, as shown in Figure \ref{fig:hae}.

$\bullet$ The HAE-X conjugates a conventional and a quadratic branch in the encoder which are fused in the intermediate layer by summation. Then, this sum is transmitted to a conventional and a quadratic branch in the decoder. Next, the yields of two branches in the decoder are summed again to produce the final model output. {This approach not only facilitates the heterogeneity in the model but also preserves its existing backbone.} 

$\bullet$ The HAE-Y is a simplification of the HAE-X, which respects the HAE-X in the encoder but only has a quadratic branch in the decoder. The reason why we put quadratic neurons into the decoder is that quadratic neurons are more capable of information extraction, which can help reconstruct the input more accurately. {HAE-Y explores how the model's outputs are affected when an encoder and a decoder possess a similar number of neurons but different types of neurons. This serves to reduce the number of parameters and complexity of heterogeneous autoencoders.}

$\bullet$ Other than using conventional and quadratic layers in parallel as the HAE-X and HAE-Y models do, the HAE-I interlaces conventional and quadratic layers in both an encoder and a decoder while keeping the symmetry. In the HAE-I model, we make the first layer a quadratic layer, as the quadratic layer is capable of extracting more information and ensures sufficient information to pass for subsequent layers. Since hardly we know which heterogeneous design is preferable, we examine them all in the experiments.

\begin{figure}[htb]
\center{\includegraphics[width=\linewidth]{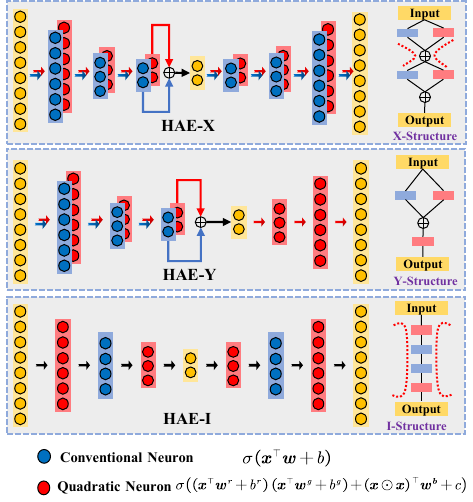}}
\caption{Three heterogeneous autoencoder designs.} 
\label{fig:hae}
\end{figure}

\textbf{Training strategies.} 
The training process of a network with quadratic neurons suffers the risk of collapse due to the high nonlinearity \cite{fan2021expressivity}. For example, the output function of a quadratic network with $L$ layers will be a polynomial of $2^L$ degrees, which is too high to keep the network stable during training. 
To fix this issue, Fan \textit{et al.} \cite{fan2021expressivity} proposed the so-called ReLinear algorithm, where the parameters in a quadratic neuron are initialized as $\boldsymbol{w}^g=0,\boldsymbol{w}^b=0,\boldsymbol{c}=0$ and $\boldsymbol{b}^g=1$, while $\boldsymbol{w}^r$ and $b^r$ follow the random initialization. Consequently, during the initialization stage, a quadratic neuron degenerates to a conventional neuron. Then, during the training stage, different learning rates are cast for $(\boldsymbol{w}^r, b^r)$ and $(\boldsymbol{w}^g, \boldsymbol{w}^b,c, b^g)$: a normal learning rate for the former and a relatively smaller learning rate for the latter. By doing so, the nonlinearity of quadratic terms are constrained, and the learned quadratic network can avoid the training instability. The schematic illustration of the ReLinear algorithm is shown in Figure \ref{fig:relinear}. 

\begin{figure}[htbp]
    \centering
    \includegraphics{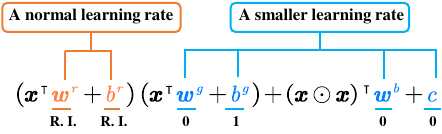}
    \caption{ReLinear: a training strategy of quadratic networks (R. I. means random initialization).} 
    \label{fig:relinear}
    \vspace{-0.5cm}
\end{figure}

\section{Experiments} 
\label{sec:Experiment}

In this section, we apply the HAE to solve the anomaly detection task, by which we can also investigate if the HAE is a powerful model as suggested by theory. The anomaly detection is to distinguish between the outliers and the normal via a threshold of contamination percentage after calculating the outlier score of each sample with a model. The experiments are twofold. First, we compare HAE with a purely conventional or quadratic autoencoder. Second, we compare the HAE with either classical or SOTA methods. The results show that the HAE outperforms its competitors on 8 benchmarks. {All quadratic neurons, except those used in an ablation study, are based on the standard quadratic neurons (Eq. \ref{qneq1})}. {Our code is available  for readers' free download and evaluation \footnote{\url{https://github.com/asdvfghg/Heterogeneous_Autoencoder_by_Quadratic_Neurons}}.}

\textbf{Tabular datasets descriptions.} We choose eight datasets from ODDS (Outlier Detection DataSets)\footnote{\url{http://odds.cs.stonybrook.edu}}: multi-dimensional point datasets and seven datasets from DAMI\footnote{\url{http://www.dbs.iﬁ.lmu.de/research/outlier-evaluation/DAMI}} \cite{campos2016evaluation}, which consist of a variety of tabular datasets from different fields. We eliminate duplicate datasets from DAMI repositories. Data in these datasets are labeled into two classes: normal and abnormal. For instance, the Wisconsin-Breast Cancer (WBC) dataset includes measurements for breast cancer cases, classified as either benign or malignant. The characteristics of all datasets are summarized in Table \ref{tb:datasets}. In all experiments, we use 80\% normal samples for training and contaminated data for testing. The test set includes normal samples that are not used in the training set and all abnormal samples. 

\begin{table}[H]
\centering
  \caption{Statistics of anomaly detection benchmark datasets.}
{
\begin{tabular}{@{}l|lccc@{}}
\hline
Repositories            & Datasets                        & \#Samples & \#Features & Outlier Ratio \\ \hline
\multirow{8}{*}{OODs} & Arrhythmia                     & 452       & 274          & 14.60\%       \\
                      & Glass                          & 214       & 9            & 4.21\%        \\
                      & Musk                           & 3062      & 166          & 3.16\%        \\
                      & Optdigits                      & 5216      & 64           & 2.88\%        \\
                      & Pendigits                      & 6870      & 16           & 2.27\%        \\
                      & Pima                           & 768       & 8            & 34.90\%       \\
                      & Verterbral                     & 240       & 6            & 12.50\%       \\
                      & Wbc                            & 378       & 30           & 5.56\%        \\ \hline
\multirow{7}{*}{DAMI} & {ALOI}       & 50000     & 27           & 3.01\%        \\
                      & {Ionosphere} & 351       & 7            & 35.89\%       \\
                      & {KDDCUP99}   & 60632     & 41           & 0.41\%        \\
                      & {Shuttle}    & 1013      & 9            & 1.28\%        \\
                      & {Waveform}   & 3443      & 21           & 2.90\%        \\
                      & {WDBC}       & 367       & 30           & 2.72\%        \\
                      & {WPBC}       & 198       & 33           & 23.7\%        \\ \hline
\end{tabular}}
  \label{tb:datasets}%
\end{table}

{\textbf{Bearing datasets descriptions}. We conduct experiments on bearing fault detection, which is a high-dimensional problem. Bearings are widely used in rotating machinery, but they are prone to damage. Anomaly detection determines whether the bearing is working under normal conditions, which is the key to bearing health management. Different from tabular data, bearing fault detection is based on long vibration signals. We use two datasets of run-to-failure bearing data to validate our method: one from the public\cite{wang2018hybrid} and the other provided by a nuclear power plant in Hainan Province, China.\\

(1) XJTU-SY dataset is collected by the Institute of Design Science and Basic Component at Xi’an Jiaotong University (XJTU) and the Changxing Sumyoung Technology Co., Ltd. (SY) \cite{wang2018hybrid}. Two accelerometers of type PCB 352C33 (sampling frequency set to 25.6 kHz) are mounted at 90° on the housing of the tested bearings, which measure the vibration of the bearing. The accelerometers record 1.28s (32768 points) per minute until the bearing fails completely. We select the Bearing2\_5 data which runs at 2250 rpm (37.5 Hz) and 11 kN load. The bearing’s lifetime is 5 h 39 min and ends with an outer race fault.\\

(2) The seawater booster pump (SBP) data are gathered from Hainan Nuclear Power Co., Ltd. The SBP transfers cooling water to steam turbines and serves as a crucial component of the auxiliary cooling water system in nuclear power plants. Given the intricate nature of nuclear plants, mechanical systems require periodic monitoring to ensure their reliability. Figure \ref{fig:sbp} illustrates that acceleration sensors are placed at four points on both the pump and the motor when performing measurement. The vibration signal for the seawater booster pumps is measured every two months, with a sampling rate of 16.8 kHz and a recording period of 0.4 seconds, which results in 6,721 data points. The faulty and healthy signals are shown in Figure \ref{fig:sbpsig}. Data in this study are collected from 2015 to 2023. During this period, a bearing fault at the outer race is observed.

\begin{figure}[htbp]
    \centering
    \includegraphics{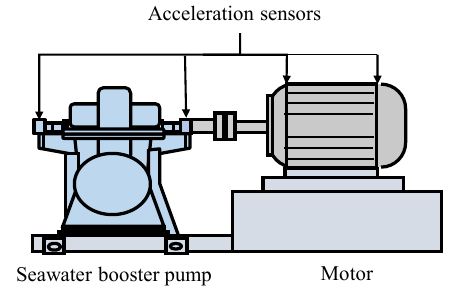}
    \caption{The measurement structure of seawater booster pump.}
    \label{fig:sbp}
\end{figure}

\begin{figure}[htbp]
    \centering
    \includegraphics{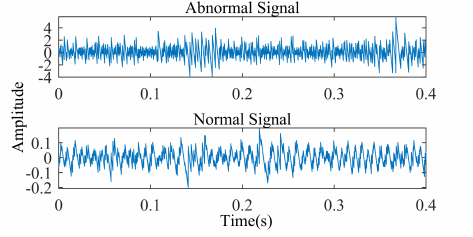}
    \caption{Abnormal and normal signals of the seawater booster pump.}
    \label{fig:sbpsig}
\end{figure}

We resample all signals to 1024 points per sample and use Fast Fourier Transform to split a single side of it into 512 points. The summary of the two datasets is in Table \ref{tab:summbearing}. Outliers in the XJTU dataset are identified based on the method described in \cite{9399476}. However, it should be noted that unlike simulation tests, the availability of faulty data from real industrial scenes is very limited. As a result, only 6 samples were available for outlier detection in the SBP dataset, which makes this task pretty challenging.}


\begin{table}[htbp]
\centering
\caption{The summary of two bearing datasets. a$\times$b represents the number of samples and sample dimensions.}

\begin{tabular}{@{}lcc@{}}
\toprule
Datasets                             & XJTU      & SBP     \\ \midrule
\#Raw Sample                         & 334$\times$32768 & 57$\times$6721 \\
\#Abnormal  Raw Sample                & 222$\times$32768 & 1$\times$6721  \\
\midrule
\#Resample                           & 8544$\times$512  & 342$\times$512 \\
\#Abnormal Resample                   & 4671$\times$512  & 6$\times$512   \\
\midrule
Outlier Ratio & 54.66\%   & 1.75\%  \\ \bottomrule
\end{tabular}
\label{tab:summbearing}
\end{table}

{\textbf{Baseline Methods.} We compare the proposed heterogeneous autoencoders with 7 baselines, SUDO \cite{zhao2021suod}, SO-GAAL \cite{liu2019generative}, DeepSVDD \cite{ruff2018deep},  DAGMM \cite{zong2018deep},  RCA \cite{liu2021rca}, AE, and QAE (an autoencoder based on quadratic neurons). All these baselines are either classical methods with many citations or published in prestigious venues. To implement them, except for autoencoders, we utilize packages in the PyOD (Python Outlier Detection) package\footnote{\url{https://github.com/yzhao062/pyod}} \cite{zhao2019pyod} or follow official scripts\footnote{\url{https://github.com/illidanlab/RCA}}. We program all autoencoder models by ourselves. We use the area under the receiver operating characteristic curve (AUC) and precision (PRE) as evaluation metrics, which are widely used in unsupervised anomaly detection \cite{aggarwal2017outlier, li2022ecod,wang2020advae}.}

\begin{table*}[htbp]
\centering
\caption{Comparison of AUCs for different anomaly detection baseline methods. Bold-faced numbers are the best and underlined numbers are the second best.}
{
\begin{tabular}{llccccccc|ccc}
\hline
Repositories            & Datasets   & SUOD           & SO-GAAL     & DeepSVDD       & DAGMM          & RCA            & AE             & QAE            & HAE-X          & HAE-Y          & HAE-I          \\ \hline
\multirow{8}{*}{OODs} & Arrhythmia & 0.770          & 0.568       & 0.704          & 0.603          & 0.806          & 0.807          & 0.807          & \textbf{0.832} & 0.816          & {\ul 0.817}    \\
                      & Glass      & \textbf{0.775} & 0.386       & 0.517          & 0.474          & 0.602          & 0.571          & 0.587          & {\ul 0.605}    & 0.608          & 0.591          \\
                      & Musk       & 0.568          & 0.438       & 0.502          & 0.903          & \textbf{1.000} & \textbf{1.000} & \textbf{1.000} & \textbf{1.000} & \textbf{1.000} & \textbf{1.000} \\
                      & Optdigits  & 0.469          & 0.367       & 0.494          & 0.290          & 0.622          & 0.652          & 0.599          & {\ul 0.772}    & \textbf{0.787} & 0.668          \\
                      & Pendigits  & 0.562          & 0.657       & 0.507          & 0.872          & 0.856          & 0.943          & 0.949          & {\ul 0.962}    & 0.943          & \textbf{0.967} \\
                      & Pima       & 0.624          & 0.292       & 0.485          & 0.531          & {\ul 0.711}          & 0.657          & 0.652          & \textbf{0.739} & 0.695          &  0.701    \\
                      & Verterbral & 0.273          & 0.634       & 0.279          & 0.487          & 0.456          &  0.568    & 0.567          & \textbf{0.574} & {\ul 0.573}          & {\ul 0.573}          \\
                      & Wbc        & \textbf{0.950} & 0.078       & 0.834          & 0.834          & 0.906          & 0.924          & {\ul 0.932}          & { 0.926}    & 0.876          & 0.915          \\ \hline
\multirow{7}{*}{DAMI} & ALOI       & {\ul 0.784}    & 0.542       & 0.531          & \textbf{1.000} & 0.554          & 0.555          & 0.555          & 0.556          & 0.547          & 0.553          \\
                      & Ionosphere & 0.850          & 0.688       & \textbf{0.979} & 0.904          & 0.917          & 0.907          & 0.906          & {\ul 0.929}    & 0.910          & 0.910          \\
                      & KDDCUP99   & 0.953          & 0.889       & 0.914          & \textbf{1.000} & {\ul 0.989}    & {\ul 0.989}    & 0.986          & {\ul 0.989}    & {\ul 0.989}    & {\ul 0.989}    \\
                      & Shuttle    & 0.493          & { 0.534} & \textbf{0.892} & {\ul 0.652}          & 0.458          & 0.365          & 0.371          & 0.473          & 0.500          & 0.495          \\
                      & Waveform   & 0.655          & 0.382       & 0.635          & 0.638          & \textbf{0.704} & 0.658          & 0.670          & {\ul 0.692}    & 0.651          & 0.680          \\
                      & WDBC       & 0.972          & 0.019       & 0.973          & 0.963          & {\ul 0.990} & \textbf{0.997}          & 0.986          & { 0.985}    & 0.978          & 0.980          \\
                      & WPBC       & {\ul 0.587}    & 0.457       & \textbf{0.703} & 0.495          & 0.460          & 0.352          & 0.449          & 0.438          & 0.443          & 0.439          \\ \hline
\multicolumn{2}{l}{Avg. $\uparrow$}           & 0.686          & 0.462       & 0.663          & 0.710          & 0.735          & 0.730          & 0.734          & \textbf{0.765} & {\ul 0.754}    & 0.752          \\
\multicolumn{2}{l}{Avg. Rank $\downarrow$}      & 6.250          & 8.438       & 7.250          & 6.438          & 4.875          & 5.500          & 5.125          & \textbf{3.063} & 4.313          & {\ul 3.750}    \\ \hline
\end{tabular}}
\label{tab:anomalyresultsauc}
\end{table*}

\begin{table*}[htbp]
\centering
\caption{Comparison of PREs for different anomaly detection baseline methods. Bold-faced numbers are the best and underlined numbers are the second best.}
{
\begin{tabular}{llccccccc|ccc}
\hline
Repositories            & Datasets   & SUOD          & SO-GAAL & DeepSVDD       & DAGMM       & RCA            & AE             & QAE         & HAE-X          & HAE-Y          & HAE-I          \\ \hline
\multirow{8}{*}{OODs} & Arrhythmia & 0.726          & 0.585   & 0.685          & 0.271       & 0.699          & 0.739          & 0.725       & \textbf{0.758} & 0.740          & {\ul 0.745}    \\
                      & Glass      & 0.593          & 0.516   & 0.573          & 0.497       & {\ul 0.596}    & 0.572          & 0.576       & 0.582          & 0.589          & \textbf{0.682} \\
                      & Musk       & \textbf{0.794} & 0.648   & 0.748          & 0.430       & 0.572          & 0.768          & 0.766       & {\ul 0.775}    & 0.773          & 0.773          \\
                      & Optdigits  & 0.448          & 0.447   & 0.473          & 0.513       & \textbf{0.566} & 0.466          & 0.447       & 0.514          & {\ul 0.515}    & 0.481          \\
                      & Pendigits  & 0.737          & 0.615   & 0.541          & 0.585       & 0.553          & 0.721          & 0.723       & \textbf{0.757} & 0.726          & {\ul 0.753}    \\
                      & Pima       & 0.638          & 0.382   & 0.511          & 0.567       & 0.541          & 0.591          & 0.636       & \textbf{0.650} & 0.629          & {\ul 0.639}    \\
                      & Verterbral & 0.422          & 0.559   & 0.427          & 0.461       & \textbf{0.656} & 0.567          & 0.495       & 0.495          & 0.574          & {\ul 0.593}    \\
                      & Wbc        & 0.745          & 0.360   & 0.657          & 0.730       & 0.621          & 0.739          & {\ul 0.750} & \textbf{0.829} & 0.704          & 0.712          \\ \hline
\multirow{7}{*}{DAMI} & ALOI       & \textbf{0.576} & 0.432   & 0.519          & 0.432       & 0.527          & 0.516          & 0.527       & 0.518          & {\ul 0.529}    & 0.528          \\
                      & Ionosphere & 0.773          & 0.631   & \textbf{0.915} & 0.787       & 0.769          & 0.808          & 0.812       & {\ul 0.832}    & 0.816          & 0.816          \\
                      & KDDCUP99   & 0.589          & 0.490   & 0.793          & 0.489       & 0.515          & 0.726          & 0.490       & 0.870          & \textbf{0.872} & {\ul 0.871}    \\
                      & Shuttle    & 0.469          & 0.482   & \textbf{0.835} & 0.495       & {\ul 0.530}    & 0.500          & 0.469       & 0.513          & 0.512          & 0.512          \\
                      & Waveform   & 0.435          & 0.517   & \textbf{0.655} & {\ul 0.600} & 0.567          & 0.527          & 0.548       & 0.557          & 0.548          & 0.547          \\
                      & WDBC       & {\ul 0.874}    & 0.437   & 0.809          & 0.714       & 0.563          & \textbf{0.993} & 0.864       & 0.859          & 0.748          & 0.747          \\
                      & WPBC       & 0.477          & 0.474   & \textbf{0.674} & {\ul 0.495} & 0.442          & 0.315          & 0.466       & 0.493          & {\ul 0.495}    & 0.486          \\ \hline
\multicolumn{2}{l}{Avg. $\uparrow$}           & 0.620          & 0.505   & 0.654          & 0.538       & 0.581          & 0.637          & 0.620       & \textbf{0.667} & 0.651          & {\ul 0.659}    \\
\multicolumn{2}{l}{Avg. Rank$\downarrow$}      & 5.313          & 8.813   & 5.375          & 7.313       & 6.125          & 5.813          & 5.938       & \textbf{3.000} & 3.813          & {\ul 3.500}    \\ \hline
\end{tabular}}
\label{tab:anomalyresultspre}
\end{table*}

{\textbf{Hyperparameters.} Suppose that the dimension {(the number of features of the input)} of the input is $d$, the structures of all autoencoders are $(d-\frac{d}{2}-\frac{d}{4}-\frac{d}{2}-d)$. The activation function is ReLU. The QAE and HAE are initialized with ReLinear \cite{fan2021expressivity}, whereas the AE is initialized with a random initialization. The dropout probability is $0.5$. The number of epochs is $100$. We use grid search for the learning rate and batch size.}\vspace{0.2cm}

\textbf{Performance comparison.} Table \ref{tab:anomalyresultsauc} and Table \ref{tab:anomalyresultspre} display the AUCs {and PREs}  of all models, where we have the following observations. First, {by comparing the AUCs,} in the majority of datasets, HAEs are the best or second best among all benchmark methods, where HAE-X has three best performances and {six} second best, HAE-Y has two best and {two} second best. Although {DeepSVDD} has three best results, its AUCs are lower on other datasets. Second, HAEs demonstrated the most consistent performance on average AUC and average rank, with all three HAE models leading.{At last, the same trend is observed when considering the PREs, with HAEs consistently outperforming other baselines in terms of average rank. While DeepSVDD showed competitive results on the DAMI repository, its average PRE is 1.3\% lower than the best-performing HAE-X model. Note that the HAEs present varying performances on different datasets, indicating that the design of a heterogeneous autoencoder is still an open question.} As a summary, our anomaly detection experiments show that the employment of diverse neurons is beneficial, supporting the earlier-developed theory for heterogeneous networks.

{Table \ref{tab:bearingresults} presents the classification results, indicating that HAE-based methods outperform other competitors in both datasets. It should be noted that DAGMM is not included as a comparison method because its training fails. The SUOD model performs well in the SBP but achieves the limited performance in XJTU dataset, while the GAAL is the opposite. The RCA model demonstrates the best precision in the XJTU dataset and the best recall in the SBP dataset. However, it fails in some other metrics. Autoencoder-based models demonstrated better results in both datasets, implying that autoencoder-based models are relatively suitable for handling higher dimensional data. In particular, HAE models are consistently better than AE and QAE, which shows that the combination of different neuron types is a wise strategy.}

\begin{table}[htbp]
\centering
\caption{Classification results of all baseline methods over two bearing fault datasets. Bold-faced numbers are better results.}

\begin{tabular}{@{}llcccc@{}}
\toprule
Datasets              & Methods  & AUC            & Pre            & Recall         & F1             \\ \midrule
\multirow{9}{*}{XJTU} & SUOD     & 0.874          & 0.714          & 0.668          & 0.686          \\
                      & SO-GAAL  & 0.958          & 0.896          & 0.724          & 0.777          \\
                      & DeepSVDD & 0.828          & 0.633          & 0.652          & 0.641          \\
                      & RCA      & 0.946          & \textbf{1.000}          & 0.528          & 0.691          \\
                      & AE       & 0.942          & 0.755          & 0.712          & 0.731          \\
                      & QAE      & 0.944          & 0.790          & 0.715          & 0.744          \\
                      & HAE-X    & 0.957          & 0.956          & \textbf{0.740} & \textbf{0.803} \\
                      & HAE-Y    & {0.958} & 0.957          & 0.737          & 0.800          \\
                      & HAE-I    & \textbf{0.961}          & {0.959} & 0.730          & 0.795          \\ \midrule
\multirow{9}{*}{SBP}  & SUOD     & \textbf{1.000} & 0.875          & 0.985          & 0.921          \\
                      & SO-GAAL  & 0.831          & 0.723          & 0.894          & 0.781          \\
                      & DeepSVDD & \textbf{1.000} & 0.700          & 0.967          & 0.769          \\
                      & RCA      & \textbf{1.000} & 0.018          & \textbf{1.000}          & 0.035          \\
                      & AE       & \textbf{1.000} & 0.850          & 0.956          & 0.911          \\
                      & QAE      & \textbf{1.000} & 0.873          & 0.963          & 0.904          \\
                      & HAE-X    & \textbf{1.000} & 0.876          & {0.997} &{0.933} \\
                      & HAE-Y    & \textbf{1.000} & \textbf{0.883} &  0.995          & \textbf{0.935}          \\
                      & HAE-I    & \textbf{1.000} & 0.881          & 0.990          & 0.932          \\ \bottomrule
\end{tabular}
\label{tab:bearingresults}
\end{table}

{\textbf{Visualization.} Because HAE variants show a better performance than a conventional one, we perform a learned embedding visualization to understand what these methods have learned. Here, with the optidigits dataset, we conduct t-SNE \cite{tsne} to visualize the latent space learned by AE and HAEs into 3D space. As shown in Figure \ref{fig:visual3d}, compared to the AE result, abnormal data show distinguishable and concentration clusters by HAEs, suggesting that a heterogeneous autoencoder better HAE better expresses the divergences between the different categories of samples.}

\begin{figure}[h]
    \centering
    \subfloat[AE]{\includegraphics[width=0.5\linewidth]{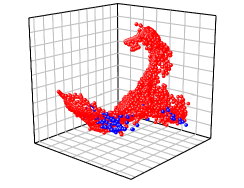}} 
    \subfloat[HAE-X]{\includegraphics[width=0.5\linewidth]{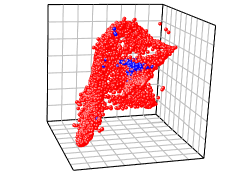}}\\
    \subfloat[HAE-Y]{\includegraphics[width=0.5\linewidth]{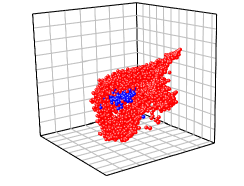}}
    \subfloat[HAE-I]{\includegraphics[width=0.5\linewidth]{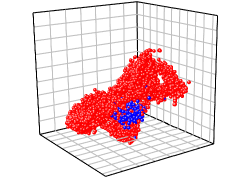}}
    \caption{The t-SNE results of four autoencoders on optidigits. The red and blue points are learned latent variables from abnormal and normal data, respectively. It is observed that the outliers are
more concentrated in HAEs than AE.}
    \label{fig:visual3d}
\end{figure}


\section{Ablation Study}

\textbf{Neuron types in HAEs}. Earlier, we compared the performance between HAEs and purely conventional or quadratic autoencoders. However, the proposed HAEs use different architectures from AE and QAE to better synergize the power of conventional and quadratic neurons. Consequently, it is not sure if the superior performance of HAEs is due to the synergy of different neurons or due to the architecture. To resolve this ambiguity, we replace neurons in HAEs to prototype homogeneous autoencoder, as Table \ref{tab:ablastruct} shows. Then, we conduct experiments on 15 anomaly detection 
datasets to compare the difference in performance between HAEs and homogeneous autoencoders using the same architecture. \\

\begin{table}[htbp]
\centering
\caption{The structure of autoencoders with different neurons. "Q" represents quadratic neurons, "C" represents conventional neurons, \fl{"Yc" denotes the HAE-Y structure using a conventional decoder, and "Ic" denotes the HAE-I structure using the conventional layers at first.}}
\begin{tabular}{@{}lcc@{}}
\toprule
 Structures       & Encoder & Decoder \\ \midrule
HAE-X   & Q+C     & Q+C     \\
QAE-X   & Q+Q     & Q+Q     \\
AE-X    & C+C     & C+C     \\
\midrule
HAE-Y   & Q+C     & Q       \\
QAE-Y   & Q+Q     & Q       \\
AE-Y    & C+C     & C       \\
\midrule
HAE-Yc & Q+C     & C       \\
HAE-I  & Q$\to$C     & C$\to$Q     \\
HAE-Ic & C$\to$Q     & Q$\to$C     \\ \bottomrule
\end{tabular}
\label{tab:ablastruct}
\end{table}

The results are presented in Tables \ref{tab:ablaacc} and \ref{tab:ablapre}. First, both AUC and precision metrics show that the performance of HAEs is superior to their homogeneous counterparts. Specifically, HAE-X achieves the highest AUC scores on 11 datasets, while HAE-Y performs best on 8 datasets and HAE-I outperforms its counterparts on 10 datasets. All quadratic neuron-based autoencoders are better than conventional ones, which underscores the powerful feature representation capabilities of quadratic neurons. At last, we observe that HAE-I with quadratic neurons in the first layer outperforms those with conventional neurons (HAE-Ic), suggesting that quadratic neurons are more adept at extracting hidden features. \\

\begin{table*}[htbp]
\centering
\caption{Comparison of AUCs for all structures. Bold-faced numbers are the best in an identical autoencoder structure with different neurons.}
{
\begin{tabular}{@{}lccc|cccc|cc@{}}
\toprule
Datasets   & HAE-X          & QAE-X          & AE-X           & HAE-Y          & QAE-Y          & AE-Y  & HAE-Yc        & HAE-Iq        & HAE-Ic        \\ \midrule
arrhythmia & \textbf{0.832} & 0.819          & 0.815          & 0.816          & \textbf{0.826} & 0.813 & 0.816          & \textbf{0.817} & 0.814          \\
glass      & \textbf{0.605} & 0.600          & 0.551          & \textbf{0.608} & 0.585          & 0.554 & 0.559          & \textbf{0.591} & 0.581          \\
musk       & 1.000          & 1.000          & 1.000          & 1.000          & 1.000          & 1.000 & 1.000          & 1.000          & 1.000          \\
optdigits  & \textbf{0.772} & 0.604          & 0.476          & \textbf{0.787} & 0.656          & 0.475 & 0.487          & \textbf{0.668} & 0.510          \\
pendigits  & \textbf{0.962} & 0.945          & 0.931          & 0.943          & \textbf{0.947} & 0.934 & 0.935          & \textbf{0.967} & 0.937          \\
pima       & \textbf{0.739} & 0.690          & 0.577          & \textbf{0.695} & 0.606          & 0.582 & 0.581          & \textbf{0.701} & 0.581          \\
verterbral & \textbf{0.574} & 0.481          & 0.571          & \textbf{0.573} & 0.561          & 0.570 & 0.571          & \textbf{0.573} & 0.572          \\
wbc        & \textbf{0.926} & 0.909          & 0.857          & \textbf{0.876} & 0.868          & 0.857 & 0.856          & \textbf{0.915} & 0.855          \\
ALOI       & \textbf{0.556} & 0.553          & 0.546          & \textbf{0.547} & 0.554          & 0.546 & \textbf{0.547} & \textbf{0.553} & 0.547          \\
Ionosphere & \textbf{0.929} & 0.920          & 0.906          & 0.910          & \textbf{0.919} & 0.905 & 0.907          & \textbf{0.910} & 0.908          \\
KDDCUP99   & 0.989          & \textbf{0.994} & 0.989          & 0.989          & 0.989          & 0.989 & 0.989          & 0.989          & 0.989          \\
Shuttle    & 0.473          & 0.385          & \textbf{0.488} & \textbf{0.500} & 0.453          & 0.491 & 0.491          & 0.495          & 0.495          \\
Waveform   & \textbf{0.692} & 0.681          & 0.643          & 0.651          & \textbf{0.688} & 0.655 & 0.647          & \textbf{0.680} & 0.657          \\
WDBC       & \textbf{0.985} & 0.981          & 0.978          & 0.978          & \textbf{0.980} & 0.979 & \textbf{0.980} & 0.980          & \textbf{0.981} \\
WPBC       & 0.438          & 0.440          & \textbf{0.445} & \textbf{0.443} & 0.439          & 0.442 & 0.442          & 0.439          & \textbf{0.442} \\ \midrule
Avg.       & \textbf{0.765} & 0.733          & 0.718          & \textbf{0.754} & 0.738          & 0.719 & 0.721          & \textbf{0.752} & 0.725          \\ \bottomrule
\end{tabular}}
\vspace{-0.2cm}
\label{tab:ablaacc}
\end{table*}

\begin{table*}[htbp]
\centering
\caption{Comparison of \fl{PREs} for all structures. Bold-faced numbers are the best in an identical autoencoder structure with different neurons.}
{
\begin{tabular}{@{}lccc|cccc|cc@{}}
\toprule
Datasets   & HAE-X          & QAE-X          & AE-X           & HAE-Y          & QAE-Y          & AE-Y           & HAE\_Yc        & HAE\_Iq        & HAE\_Ic        \\ \midrule
arrhythmia & \textbf{0.758} & 0.742          & 0.729          & \textbf{0.740} & 0.735          & 0.724          & 0.736          & \textbf{0.745} & 0.734          \\
glass      & \textbf{0.582} & 0.576          & 0.556          & 0.589          & \textbf{0.636} & 0.574          & 0.574          & \textbf{0.682} & 0.572          \\
musk       & 0.775          & \textbf{0.776} & 0.770          & 0.773          & \textbf{0.917} & 0.771          & 0.773          & \textbf{0.773} & 0.770          \\
optdigits  & \textbf{0.514} & 0.443          & 0.431          & \textbf{0.515} & 0.433          & 0.432          & 0.431          & \textbf{0.481} & 0.431          \\
pendigits  & \textbf{0.757} & 0.717          & 0.712          & 0.726          & \textbf{0.818} & 0.711          & 0.715          & \textbf{0.753} & 0.715          \\
pima       & \textbf{0.650} & 0.626          & 0.557          & \textbf{0.629} & 0.573          & 0.555          & 0.555          & \textbf{0.639} & 0.553          \\
verterbral & 0.495          & 0.504          & \textbf{0.582} & 0.574          & 0.572          & 0.586          & \textbf{0.589} & \textbf{0.593} & 0.589          \\
wbc        & \textbf{0.829} & 0.723          & 0.699          & \textbf{0.704} & 0.695          & 0.700          & 0.703          & \textbf{0.712} & 0.700          \\
ALOI       & \textbf{0.518} & \textbf{0.518} & 0.515          & \textbf{0.529} & 0.517          & 0.516          & 0.516          & \textbf{0.528} & 0.516          \\
Ionosphere & \textbf{0.832} & 0.827          & 0.812          & \textbf{0.816} & 0.815          & 0.806          & 0.806          & \textbf{0.816} & 0.807          \\
KDDCUP99   & \textbf{0.870} & 0.713          & 0.724          & \textbf{0.872} & 0.727          & 0.729          & \textbf{0.872} & \textbf{0.871} & 0.726          \\
Shuttle    & 0.513          & 0.509          & \textbf{0.518} & 0.512          & 0.511          & \textbf{0.519} & 0.469          & 0.512          & \textbf{0.522} \\
Waveform   & \textbf{0.557} & 0.546          & 0.537          & 0.548          & \textbf{0.558} & 0.541          & 0.544          & \textbf{0.547} & 0.537          \\
WDBC       & \textbf{0.859} & 0.747          & 0.747          & 0.748          & 0.749          & 0.746          & \textbf{0.855} & 0.747          & 0.747          \\
WPBC       & 0.493          & 0.487          & \textbf{0.494} & \textbf{0.495} & 0.490          & 0.484          & 0.487          & 0.486          & \textbf{0.488} \\ \midrule
Avg.       & \textbf{0.667} & 0.630          & 0.626          & \textbf{0.651} & 0.650          & 0.626          & 0.642          & \textbf{0.659} & 0.627          \\ \bottomrule
\end{tabular}}
\vspace{-0.2cm}
\label{tab:ablapre}
\end{table*}

\begin{table*}[h]
\centering
\caption{Comparison of AUC for different quadratic functions on some datasets. \textbf{Bold-faced} numbers are the best compared in every structure.}
\scalebox{0.9}{
\begin{tabular}{@{}llccccccc|c@{}}
\toprule
Methods              & Quadratic Function & {Arrhythmia} & {Glass} & {Optdigits} & {ALOI} &{Shuffle} & {Waveform} & {XJTU} & {Avg.} \\ \midrule
\multirow{3}{*}{HAE-X} & $(\x^\top\w^r+b^r) ( \x^\top\w^g+b^g)$               & 0.817                          & 0.591                     & 0.627                         & 0.552                    & 0.471                       & 0.682                        & 0.942                    & 0.669                    \\
                       & $\x^\top\w^r +(\x \odot \x)^\top \w^b +c$             & 0.830                          & 0.580                     & 0.676                         & 0.554                    & 0.471                       & 0.687                        & 0.945                    & 0.678                    \\
                       & Standard            & \textbf{0.832}                 & \textbf{0.605}            & \textbf{0.772}                & \textbf{0.556}           & \textbf{0.473}              & \textbf{0.692}               & \textbf{0.957}           & \textbf{0.698}           \\ \midrule
\multirow{3}{*}{HAE-Y} & $(\x^\top\w^r+b^r) ( \x^\top\w^g+b^g)$               & 0.820                          & 0.578                     & 0.659                         & 0.553                    & 0.453                       & 0.689                        & 0.942                    & 0.671                    \\
                       & $\x^\top\w^r +(\x \odot \x)^\top \w^b +c$             & \textbf{0.829}                 & 0.571                     & 0.661                         & \textbf{0.554}                    & 0.452                       & \textbf{0.695}               & 0.943                    & 0.672                    \\
                       & Standard             & 0.816                          & \textbf{0.608}            & \textbf{0.787}                & {0.547}           & \textbf{0.500}              & 0.651                        & \textbf{0.958}           & \textbf{0.695}           \\ \midrule
\multirow{3}{*}{HAE-I} & $(\x^\top\w^r+b^r) ( \x^\top\w^g+b^g)$               & \textbf{0.818}                 & 0.567                     & 0.627                         & 0.551                    & 0.450                       & 0.671                        & 0.944                    & 0.661                    \\
                       & $\x^\top\w^r +(\x \odot \x)^\top \w^b +c$             & 0.815                          & 0.568                     & 0.628                         & \textbf{0.553}           & 0.450                       & 0.670                        & \textbf{0.966}           & 0.664                    \\
                       & Standard             & 0.817                          & \textbf{0.591}            & \textbf{0.668}                & \textbf{0.553}           & \textbf{0.495}              & \textbf{0.680}               & 0.961                    & \textbf{0.681}           \\ \bottomrule
\end{tabular}}
\vspace{-0.2cm}
\label{tab:quadratic_auc}
\end{table*}

\begin{table*}[h]
\centering
\caption{Comparison of precision for different quadratic functions on some datasets. \textbf{Bold-faced} numbers are the best compared in every structure.}
\scalebox{0.9}{
\begin{tabular}{@{}llccccccc|c@{}}
\toprule
Methods              & Quadratic Function & {Arrhythmia} & {Glass} & \multicolumn{1}{l}{Optdigits} & {ALOI} & \multicolumn{1}{l}{Shuffle} & {Waveform} & {XJTU} & {Avg.} \\ \midrule
\multirow{3}{*}{HAE-X} & $(\x^\top\w^r+b^r) ( \x^\top\w^g+b^g)$                & 0.716                          & \textbf{0.634}            & 0.457                         & 0.516                    & 0.517                       & 0.548                        & 0.747                    & 0.591                    \\
                       & $\x^\top\w^r +(\x \odot \x)^\top \w^b +c$             & 0.742                          & 0.568                     & 0.481                         & 0.515                    & \textbf{0.523}              & 0.554                        & 0.818                    & 0.600                    \\
                       & Origin             & \textbf{0.758}                 & 0.582                     & \textbf{0.514}                & \textbf{0.581}           & 0.513                       & \textbf{0.557}               & \textbf{0.956}           & \textbf{0.637}           \\ \midrule
\multirow{3}{*}{HAE-Y} & $(\x^\top\w^r+b^r) ( \x^\top\w^g+b^g)$                & \textbf{0.753}                 & 0.572                     & 0.460                         & 0.516                    & \textbf{0.522}              & 0.550                        & 0.746                    & 0.588                    \\
                       &$\x^\top\w^r +(\x \odot \x)^\top \w^b +c$             & 0.744                          & 0.572                     & 0.475                         & 0.515                    & 0.520                       & \textbf{0.557}               & 0.852                    & 0.605                    \\
                       & Origin             & 0.740                          & \textbf{0.589}            & \textbf{0.515}                & \textbf{0.529}           & 0.512                       & 0.548                        & \textbf{0.957}           & \textbf{0.627}           \\ \midrule
\multirow{3}{*}{HAE-I} & $(\x^\top\w^r+b^r) ( \x^\top\w^g+b^g)$                & 0.741                          & 0.573                     & 0.455                         & 0.514                    & \textbf{0.551}                        & 0.539                        & 0.728                    & 0.586                    \\
                       &$\x^\top\w^r +(\x \odot \x)^\top \w^b +c $            & 0.742                          & 0.573                     & 0.453                         & 0.514                    & 0.511                       & 0.541                        & 0.923                    & 0.608                    \\
                       & Origin             & \textbf{0.745}                 & \textbf{0.682}            & \textbf{0.481}                & \textbf{0.528}           & {0.512}              & \textbf{0.547}               & \textbf{0.959}           & \textbf{0.636}           \\ \bottomrule
\end{tabular}}
\label{tab:quadratic_pre}
\end{table*}

\textbf{The form of quadratic neurons}. Quadratic neurons have different variants, such as only keeping the power terms and replacing the interaction term with a linear term. What we use in a quadratic neuron is the standard form of the quadratic function. But is it suitable for anomaly detection? Here, we investigate if the standard form of quadratic neurons fits. We have chosen several datasets for comparison, with the results for AUC and precision presented in Tables \ref{tab:quadratic_auc} and \ref{tab:quadratic_pre}, respectively. We highlight that models using standard quadratic neurons display superior performance in the majority of cases. This outcome strongly suggests that the standard quadratic neuron, with its inclusion of both power and inner-product terms, possesses the best representation capabilities. Additionally, it is worth noting that models utilizing the power term generally perform better than those omitting it. This observation highlights the important role of the power term within the quadratic neuron.

\begin{table}[H]
\caption{Comparison of AUC for different network structures on optidigits. \textbf{Bold-faced} numbers are the best result compared in every column. For hidden neurons,  V1 represents [\textit{dim(x)}/2-\textit{dim(x)}/4], V2 is [\textit{dim(x)}/2-\textit{dim(x)}/3-\textit{dim(x)}/4], V3 is [\textit{dim(x)}/2-\textit{dim(x)}/3-\textit{dim(x)}/4-\textit{dim(x)}/4] and V4 is [\textit{dim(x)}/2-\textit{dim(x)}/3-\textit{dim(x)}/3-\textit{dim(x)}/4-\textit{dim(x)}/4].}
\scalebox{0.85}{
\begin{tabular}{@{}cccccc@{}}
\toprule
  & AE  & QAE & HAE-X      & HAE-Y      & HAE-I      \\ \midrule
V1 &\textbf{0.652\footnotesize$\pm$0.017} & \textbf{0.599\footnotesize$\pm$0.023}  & \textbf{0.680\footnotesize$\pm$0.054} & \textbf{0.681\footnotesize$\pm$0.056} & 0.617\footnotesize$\pm$0.031 \\
V2 &0.613\footnotesize$\pm$0.010 & 0.528\footnotesize$\pm$0.045 & 0.669\footnotesize$\pm$0.043 & 0.653\footnotesize$\pm$0.027 & \textbf{0.636\footnotesize$\pm$0.023} \\
V3 &0.582\footnotesize$\pm$0.028 & 0.447\footnotesize$\pm$0.054 & 0.671\footnotesize$\pm$0.036 & 0.666\footnotesize$\pm$0.068 & 0.625\footnotesize$\pm$0.026 \\
V4 &0.572\footnotesize$\pm$0.008 &0.316\footnotesize$\pm$0.052 & 0.647\footnotesize$\pm$0.029 & 0.653\footnotesize$\pm$0.040 & 0.613\footnotesize$\pm$0.041 \\ \bottomrule
\end{tabular}}
\label{tab:structure}
\end{table}

\section{Analysis Experiments}



\textbf{Network depth.} As shown in Table \ref{tab:structure}, we investigate the performance of HAEs in response to different depths on the optidigits dataset. Note that increasing network layers will not tend to a better performance when the scale of the neural network is large enough. It may lead to over-fitting. Therefore, all autoencoders are based on V1 ([\textit{dim(x)}/2-\textit{dim(x)}/4]) structure in our follows experiments.

{\textbf{Network width.} How to find the best width arrangement for autoencoders is still an open question. If the width is large, the compressing effect is compromised, which causes the encoder to fail to learn the most representative features. If the width is small, the learned features are insufficient to completely represent the original data, which makes the detection based on these features inaccurate. We conduct an experiment to evaluate the validity of different width arrangements. The results are presented in Table \ref{tab:neurons}. While our recommended structure exhibits superior performance in the majority of cases, we also note that the results are quite similar. We believe that, in the context of anomaly detection where large-scale datasets are often unavailable, the number of neurons employed does not significantly impact results.}

\begin{table}[H]
\caption{Comparison of HAEs using different width arrangements in several datasets. Bold-faced numbers are better.}
\centering
\begin{tabular}{@{}llcccccc@{}}
\toprule
\multicolumn{2}{l}{Dataset}        & \multicolumn{2}{c}{Arrhythmia}  & \multicolumn{2}{c}{Shuffle}     & \multicolumn{2}{c}{XJTU}        \\ \midrule
Method                 & Structure & AUC            & PRE            & AUC            & PRE            & AUC            & PRE            \\ \midrule
\multirow{3}{*}{HAE-X} & d-d/2-d/4 & \textbf{0.832} & \textbf{0.758} & \textbf{0.473} & 0.513          & \textbf{0.957} & \textbf{0.956} \\
                       & d-d/4-d/8 & 0.820          & 0.744          & 0.435          & 0.511          & 0.946          & 0.925          \\
                       & d-d-d/2   & 0.828          & 0.741          & 0.452          & \textbf{0.520} & 0.945          & 0.913          \\ \midrule
\multirow{3}{*}{HAE-Y} & d-d/2-d/4 & 0.816          & 0.740          & \textbf{0.500} & \textbf{0.512} & \textbf{0.958} & \textbf{0.957} \\
                       & d-d/4-d/8 & 0.816          & 0.732          & 0.449          & 0.511          & 0.948          & 0.935          \\
                       & d-d-d/2   & \textbf{0.827} & \textbf{0.741} & 0.421          & 0.509          & 0.945          & 0.908          \\ \midrule
\multirow{3}{*}{HAE-I} & d-d/2-d/4 & \textbf{0.817} & \textbf{0.745} & \textbf{0.495} & \textbf{0.512} & \textbf{0.961} & \textbf{0.959} \\
                       & d-d/4-d/8 & \textbf{0.817} & 0.738          & 0.454          & 0.511          & 0.962          & 0.937          \\
                       & d-d-d/2   & 0.814          & 0.737          & 0.445          & 0.508          & 0.958          & 0.947          \\ \bottomrule
\end{tabular}
\label{tab:neurons}
\end{table}

\section{Conclusions}
\label{sec:Conclusions}

In this article, we have theoretically shown that a heterogeneous network is more powerful and efficient by constructing a function that is easy to approximate for a one-hidden-layer heterogeneous network but difficult to approximate by a one-hidden-layer conventional or quadratic network. Moreover, we have proposed a novel type of heterogeneous autoencoders using quadratic and conventional neurons. Lastly, anomaly detection experiments have confirmed that a heterogeneous autoencoder delivers competitive performance relative to homogeneous autoencoders and other baselines. Future work will be designing other heterogeneous models and applying them to more real-world problems. 

\section*{Acknowledgement}
We would like to extend our heartfelt gratitude to Mr. Jianbin Zhu, esteemed Dean of the Vibration and Fault Diagnosis Laboratory at Hainan Nuclear Power Co., Ltd. His invaluable assistance in data collection has significantly contributed to our work. We deeply appreciate his generosity and support.

\small
\bibliographystyle{IEEEtran}
\bibliography{bio.bib}

\vfill

\end{document}